\documentclass[10pt]{article}

\usepackage{research_preprint}
\usepackage{lmodern}
\usepackage{amsmath,amssymb,amsthm}
\usepackage{array}
\usepackage{booktabs}
\usepackage{graphicx}
\usepackage{tikz}
\usepackage{xcolor}
\usepackage{hyperref}
\hypersetup{hidelinks}

\usepackage{url}
\usepackage{microtype}
\usepackage{enumitem}

\definecolor{navy}{RGB}{0,56,118}
\definecolor{crimson}{RGB}{165,28,48}
\definecolor{slate}{RGB}{90,100,110}

\newtheorem{theorem}{Theorem}
\newtheorem{proposition}{Proposition}
\newtheorem{lemma}{Lemma}
\newtheorem{corollary}{Corollary}

\newtheorem{assumption}{Assumption}

\newcommand{\one}{\mathbf{1}}
\newcommand{\TV}{\mathrm{TV}}
\newcommand{\supp}{\mathrm{supp}}
\DeclareMathOperator{\diag}{diag}

\title{Support Topology and Gradient Mixing in Sinkhorn Layers}
\author{Dylan B. Forde\\Independent Researcher\\\texttt{forde.dylan@gmail.com}}
\date{}

\hypersetup{pdfauthor={Dylan B. Forde},pdftitle={Support Topology and Gradient Mixing in Sinkhorn Layers}}
\begin{document}
\maketitle

\begin{abstract}
Sinkhorn attention layers normalize weights over a fixed sparse mask by alternating row and column scaling. When training these layers, sparse masks are usually judged by their edge count, graph diameter, or forward reachability. We show that the sparsity pattern and marginal budgets also govern how gradients flow backward through the scaling steps. Differentiating one row--column cycle on column log-potentials yields an explicit row-stochastic Markov kernel built from the two half-step transport plans, which reduces to a reversible two-step random walk at a balanced fixed point. Because potentials are defined only up to additive constants on each connected component, reverse-mode gradients propagate as zero-sum vectors whose total-variation norm contracts at a rate given by the kernel's Dobrushin coefficient. This coefficient bounds both unrolled backpropagation and implicit-differentiation truncation error. Crucially, graph connectivity alone does not prevent gradient trapping: we prove a face-lattice dichotomy showing that a mask and marginals guarantee score-uniform one-step contraction if and only if every nonempty face of the transportation polytope retains pairwise two-hop column overlap; otherwise, large learned scores can push the contraction rate arbitrarily close to one. We extend this criterion to multi-layer schedules, prove exact one-sweep mixing for coordinate-factorized and butterfly masks, and derive closed-form bounds for expanders, local bands, and global hub tokens. Deterministic numerical checks verify our finite-difference adjoints, canonical sparse masks, adversarial score limits, and spectral formulas.
\end{abstract}

\section{Introduction}

Attention builds an output by weighting and combining values. In standard softmax attention, each query normalizes its row independently, so many queries can concentrate on the same key \citep{vaswani2017attention}. Differentiable optimal transport (OT) layers add a second constraint: every key also has a target capacity, and Sinkhorn iteration alternates row and column rescalings to balance both budgets \citep{cuturi2013sinkhorn,tay2020sparse,sander2022sinkformers}. On long sequences, computing a dense all-pairs transport plan is too expensive, so practical layers restrict attention to a fixed sparse support mask---such as a local sliding window, a random expander, or a small set of global hub tokens \citep{beltagy2020longformer,zaheer2020bigbird,forde2026blockwise}.

Sparse attention masks are usually evaluated in the forward pass: how many edges they keep, how many hops separate distant tokens, or how many FLOPs they save. Once the support mask is fixed, however, the scaling iterations also define a structured linear map in the backward pass. Whether training uses a short unrolled refinement tail \citep{forde2026blockwise} or implicit differentiation at a fixed point \citep{luise2018differential,eisenberger2022unified,pauwels2023derivatives}, reverse-mode gradients through the dual potentials move according to a two-step Markov kernel on columns. Its Dobrushin contraction coefficient determines whether gradient errors wash out across iterations or stay trapped in isolated modes.

The main surprise is that a well-connected graph can still fail to mix gradients once edge scores are learned. Even if a mask has a global hub row connected to every token, strong score preferences can starve the hub edges and push the transport plan onto a disconnected boundary face of the feasible polytope. At that boundary, two-hop column neighborhoods stop overlapping, and the contraction rate approaches one. Preventing that collapse requires either keeping realized edge masses bounded away from zero or sizing the marginal budgets so the hub cannot empty out.

\begin{figure}[t]
  \centering
  \resizebox{\linewidth}{!}{%
  \begin{tikzpicture}[>=stealth, font=\small]
    \node[anchor=west, font=\bfseries] at (-0.2, 3.12) {(a) Two-hop column walk $M = C^\top R$};
    \draw[rounded corners, draw=black!25, fill=black!2] (-0.2,-0.32) rectangle (6.4,2.80);

    \node[draw, circle, fill=white, minimum size=0.62cm] (i1) at (1.9, 2.08) {$i_1$};
    \node[draw, circle, fill=white, minimum size=0.62cm] (i2) at (4.5, 2.08) {$i_2$};
    \node[anchor=east, font=\scriptsize, text=slate] at (1.35, 2.08) {Rows};

    \node[draw, circle, fill=white, minimum size=0.62cm] (a) at (1.1, 0.48) {$\alpha$};
    \node[draw, circle, fill=white, minimum size=0.62cm] (g) at (3.2, 0.48) {$\gamma$};
    \node[draw, circle, fill=white, minimum size=0.62cm] (b) at (5.3, 0.48) {$\beta$};
    \node[anchor=east, font=\scriptsize, text=slate] at (0.6, 0.48) {Cols};

    \draw[thick, color=navy, ->] (a) -- node[left, font=\scriptsize] {$C_{i_1\alpha}$} (i1);
    \draw[thick, color=navy, ->] (i1) -- node[right, pos=0.35, font=\scriptsize] {$R_{i_1\gamma}$} (g);
    \draw[thick, color=crimson, ->] (b) -- node[right, font=\scriptsize] {$C_{i_2\beta}$} (i2);
    \draw[thick, color=crimson, ->] (i2) -- node[left, pos=0.35, font=\scriptsize] {$R_{i_2\gamma}$} (g);

    \node[align=center, font=\scriptsize, text=black!80] at (3.2, -0.06)
      {Rows $\alpha$ and $\beta$ of $M$ share positive mass at $\gamma$};

    \node[anchor=west, font=\bfseries] at (7.1, 3.12) {(b) Two-hop overlap and contraction $\tau(M)$};
    \draw[rounded corners, draw=black!25, fill=black!2] (7.1,-0.32) rectangle (15.45,2.80);

    \node[anchor=west, font=\scriptsize\bfseries] at (7.35, 2.48) {Local band ($r=1, n=10$): disjoint 2-hop sets ($\tau=1$)};
    \foreach \k in {1,...,10} {
      \node[draw, minimum width=0.46cm, minimum height=0.35cm, fill=white, font=\scriptsize] (t\k) at (7.18 + 0.72*\k, 1.92) {\k};
    }
    \draw[very thick, color=navy] (7.90-0.27, 1.69) rectangle (10.06+0.27, 2.15);
    \draw[very thick, color=crimson] (11.50-0.27, 1.69) rectangle (14.38+0.27, 2.15);
    \node[font=\scriptsize, text=navy] at (8.98, 1.44) {from col $2$};
    \node[font=\scriptsize, text=crimson] at (12.94, 1.44) {from col $8$};

    \node[anchor=west, font=\scriptsize\bfseries] at (7.35, 1.04) {With shared hub row $h$: common mass ($\tau \le 1 - \gamma_{\mathrm{shared}}$)};
    \node[draw, rounded corners, fill=navy!10, minimum width=1.15cm, minimum height=0.38cm, font=\scriptsize\bfseries] (hub) at (11.1, 0.48) {hub $h$};
    \node[draw, circle, fill=white, minimum size=0.46cm, font=\scriptsize] (c2) at (8.4, 0.48) {$2$};
    \node[draw, circle, fill=white, minimum size=0.46cm, font=\scriptsize] (c8) at (13.8, 0.48) {$8$};
    \draw[thick, color=navy, ->] (c2) -- (hub);
    \draw[thick, color=crimson, ->] (c8) -- (hub);
    \node[font=\scriptsize, text=slate] at (11.1, -0.06) {all columns overlap through $h$ if hub edges stay positive};
  \end{tikzpicture}
  }
  \caption{\textbf{Support structure and two-hop gradient mixing.} \textbf{(a)} Differentiating one Sinkhorn cycle produces the column-to-column kernel $M = C^\top R$, which moves from column $\alpha$ to an active row $i$ and back to column $\gamma$. \textbf{(b)} In a narrow local band, distant columns have disjoint two-hop neighborhoods, leaving $\tau(M)=1$. Adding a shared hub row creates common overlap as long as the hub keeps positive mass.}
  \label{fig:two-hop-walk}
\end{figure}

\paragraph{Contributions.}
We make four contributions:
\begin{enumerate}[leftmargin=1.5em]
  \item \textbf{Fixed-support differential calculus (Section~\ref{sec:calculus}).} We derive the exact forward linearization and reverse-mode equations for a finite Sinkhorn cycle on a fixed support, distinguishing the two half-step plans $P_t^r \neq P_t^c$ out of equilibrium. Using Dobrushin's coefficient on zero-sum dual gradients, we bound both unrolled tail backpropagation and implicit-differentiation truncation error.
  \item \textbf{Equal-score baselines and score sensitivity (Section~\ref{sec:baselines}).} When scores are equal on a biregular mask, the derivative kernel is the two-step simple random walk on the support graph. We bound how contraction degrades under plan distortion and show that the dimension of scores that can change the plan equals the graph's cycle rank.
  \item \textbf{Face-lattice dichotomy for score-uniform contraction (Section~\ref{sec:face-lattice}).} We prove that a mask and marginals force one-step contraction uniformly over all finite scores if and only if every nonempty face of the transportation polytope has pairwise two-hop column overlap. We pair this tight condition with checkable lower bounds via Doeblin minorization and marginal-forced global hubs.
  \item \textbf{Multi-layer schedules, canonical masks, and verification (Sections~\ref{sec:schedules}--\ref{sec:verification}).} We extend the face-lattice characterization to ordered sequences of sparse layers, prove exact one-sweep mixing for coordinate-factorized and hypercube butterfly schedules whose individual layers do not contract ($\tau = 1$), and derive spectral and closed-form bounds for expanders, local bands, and global hubs. Deterministic numerical checks verify our finite-difference adjoints, canonical masks, adversarial score limits, and spectral formulas.
\end{enumerate}

\paragraph{Scope.}
All results apply to fixed active supports, finite edge scores, and compatible positive marginals. We study the Sinkhorn scaling map itself; when surrounding neural layers inject upstream gradients, our bounds apply after projecting those terms onto the zero-sum space and bounding their norms separately. Downstream language or alignment benchmarks are outside the scope of this paper.

\paragraph{Relation to prior work.}
Matrix scaling dates to \citet{sinkhornknopp1967} and \citet{franklin1989scaling}, with support conditions surveyed in \citet{idel2016review}, local convergence rates analyzed via $C^\top R$ in \citet{knight2008sinkhorn}, and phase transitions studied in \citet{he2026phase}. Entropic regularization \citep{cuturi2013sinkhorn,peyre2019computational} and warm-start rules \citep{thornton2023initialization} made Sinkhorn practical inside neural architectures, while \citet{luise2018differential}, \citet{eisenberger2022unified}, \citet{pauwels2023derivatives}, \citet{greco2023sinkhorn}, and \citet{forde2026blockwise} analyzed unrolled, implicit, and tail-refinement differentiation. On the architectural side, sparse and register-augmented attention \citep{child2019sparse,beltagy2020longformer,zaheer2020bigbird,roy2020routing,darcet2023registers} and sparse mixture-of-experts routing \citep{shazeer2017outrageously,fedus2022switch} motivate our mask families, alongside doubly-stochastic attention and OT kernels \citep{tay2020sparse,sander2022sinkformers,shahbazi2025lotformer,flashsinkhorn2026}, constrained hyper-connections \citep{mhc2025}, and topological message passing \citep{bodnar2022neural_sheaf,hajij2022topological}. Our contraction tools come from Markov chain ergodicity and stochastic matrix products \citep{doeblin1938expose,dobrushin1956central,wolfowitz1963products,seneta2006nonnegative,gaubert2013dobrushin,levin2017markov}, and our score-uniform characterization builds on transportation polytopes \citep{kleewitzgall1968faces,brualdi2006combinatorial,deloera2013transportation} and zero-temperature linear programming limits \citep{cominetti1994asymptotic,weed2018explicit}.

\section{How Sparse Sinkhorn Gradients Move}
\label{sec:calculus}

Consider a bipartite sparsity pattern $\Omega \subseteq [m]\times[n]$ with row neighborhoods $\Omega_i = \{j : (i,j)\in\Omega\}$ and column neighborhoods $\Omega^j = \{i : (i,j)\in\Omega\}$. Each row $i$ has an emitting budget $a_i > 0$ and each column $j$ has an absorbing budget $b_j > 0$, with total mass $T = \one^\top a = \one^\top b$. Active edges $(i,j) \in \Omega$ carry finite scores $s_{ij} \in \mathbb{R}$ (for logits $\theta_{ij}$ at temperature $\varepsilon > 0$, $s_{ij} = \theta_{ij}/\varepsilon$). Table~\ref{tab:notation} summarizes the notation used throughout the paper.

\begin{table}[t]
  \centering
  \small
  \caption{\textbf{Notation used in the main text.}}
  \label{tab:notation}
  \begin{tabular}{@{}ll@{}}
    \toprule
    Symbol & Meaning \\
    \midrule
    $\Omega \subseteq [m]\times[n]$ & Fixed active support mask; row/column neighborhoods $\Omega_i, \Omega^j$ \\
    $a \in \mathbb{R}^m_{>0},\; b \in \mathbb{R}^n_{>0}$ & Compatible positive row and column marginals with $\one^\top a = \one^\top b = T$ \\
    $s_{ij} \in \mathbb{R}$ & Active-edge score ($\theta_{ij}/\varepsilon$ under temperature $\varepsilon > 0$) \\
    $u, v;\; P$ & Dual log-potentials and transport plan $P_{ij} = \exp(s_{ij}+u_i+v_j)$ on $\Omega$ \\
    $P_t^r, P_t^c;\; R_t, C_t$ & Row- and column-normalized half-step plans; $R_t = \diag(a)^{-1}P_t^r$, $C_t = P_t^c\diag(b)^{-1}$ \\
    $M_t = C_t^\top R_t;\; M$ & Finite-cycle column Jacobian ($M_t$) and balanced fixed-point kernel ($M$) \\
    $Q, Q^*;\; \|\eta\|_{\TV}$ & Column potentials mod constants; zero-sum adjoints $\one^\top \eta = 0$ with $\|\eta\|_{\TV} = \frac{1}{2}\|\eta\|_1$ \\
    $\tau(M)$ & Dobrushin contraction coefficient $\frac{1}{2}\max_{p,q}\|M_{p\cdot} - M_{q\cdot}\|_1 \in [0,1]$ \\
    $\Pi(a,b;\Omega)$ & Transportation polytope of nonnegative plans on $\Omega$ with marginals $(a,b)$ \\
    \bottomrule
  \end{tabular}
\end{table}

Given dual log-potentials $(u,v) \in \mathbb{R}^m \times \mathbb{R}^n$, the active-edge plan is
\[
  P_{ij}(u,v;s) = \exp(s_{ij} + u_i + v_j) \quad ((i,j)\in\Omega),
  \qquad P_{ij} = 0 \quad ((i,j)\notin\Omega).
\]
Adding a constant $c_\Gamma$ to $u$ and subtracting $c_\Gamma$ from $v$ on any connected component $\Gamma$ of $\Omega$ leaves $P$ unchanged. A Sinkhorn solution satisfies $P\one = a$ and $P^\top \one = b$.

\begin{assumption}[Fixed support and strict feasibility]
\label{ass:fixed-support-feasible}
The support $\Omega$ is fixed during differentiation, active scores $s_{ij}$ are finite, and every connected component $\Gamma = (I_\Gamma, J_\Gamma)$ of $\Omega$ has equal row and column marginal mass: $\sum_{i \in I_\Gamma} a_i = \sum_{j \in J_\Gamma} b_j$. When strict edgewise positivity is invoked, $\Pi(a,b;\Omega)$ contains a point positive on every edge of $\Omega$.
\end{assumption}

\subsection{One-Cycle Linearization and Reverse-Mode Equations}

Starting from an input column potential $v$, one row--column scaling cycle at step $t$ computes
\begin{align}
  u_i^+(v; s, a) &= \log a_i - \log \sum_{j \in \Omega_i} \exp(s_{ij} + v_j),
  & (P_t^r)_{ij} &= \exp(s_{ij} + u_i^+ + v_j), \label{eq:row-half-step}\\
  v_j^+(u^+; s, b) &= \log b_j - \log \sum_{i \in \Omega^j} \exp(s_{ij} + u_i^+),
  & (P_t^c)_{ij} &= \exp(s_{ij} + u_i^+ + v_j^+), \label{eq:column-half-step}
\end{align}
for $(i,j)\in\Omega$. The first update enforces $P_t^r\one = a$, and the second enforces $(P_t^c)^\top\one = b$. Normalizing each half-step by its enforced marginal defines the stochastic matrices
\[
  R_t = \diag(a)^{-1} P_t^r,
  \qquad
  C_t = P_t^c \diag(b)^{-1},
  \qquad
  M_t = C_t^\top R_t \in \mathbb{R}^{n\times n}.
\]
Because rows of $R_t$ and columns of $C_t$ sum to one, $M_t$ is always row-stochastic ($M_t\one = \one$). Out of equilibrium, $P_t^r \neq P_t^c$, so $M_t$ need not be reversible or preserve $b$. At a balanced fixed point $P_t^r = P_t^c = P$, we write $R = \diag(a)^{-1}P$, $C = P\diag(b)^{-1}$, and $M = C^\top R = \diag(b)^{-1}P^\top \diag(a)^{-1}P$.

Write $d_a = \dot a \oslash a$ and $d_b = \dot b \oslash b$ for log-marginal perturbations, and define the half-step score averages $r_t(\dot s)_i = \sum_{j\in\Omega_i} (R_t)_{ij}\dot s_{ij}$ and $c_t(\dot s)_j = \sum_{i\in\Omega^j} (C_t)_{ij}\dot s_{ij}$.

\begin{proposition}[First-order linearization of a Sinkhorn cycle]
\label{prop:half-step-linearization}
Under Assumption~\ref{ass:fixed-support-feasible}, the half-steps \eqref{eq:row-half-step}--\eqref{eq:column-half-step} at cycle $t$ satisfy $\dot u^+ = d_a - R_t \dot v - r_t(\dot s)$ and $\dot v^+ = d_b - C_t^\top \dot u^+ - c_t(\dot s)$, so
\begin{equation}
\label{eq:full-linearization-with-sources}
  \dot v^+ = M_t \dot v + d_b - C_t^\top d_a + C_t^\top r_t(\dot s) - c_t(\dot s).
\end{equation}
At a balanced fixed point, $R_t$ and $C_t$ coincide with $R$ and $C$.
\end{proposition}

Differentiating each log-sum-exp replaces it with an average under that half-step's normalized plan, and composing the two negative signs yields the nonnegative transition matrix $M_t = C_t^\top R_t$. Transposing \eqref{eq:full-linearization-with-sources} immediately gives the one-cycle vector-Jacobian product (VJP).

\begin{proposition}[Reverse-mode equations for one cycle]
\label{prop:reverse-mode-cycle}
For an output adjoint $\eta^+ \in \mathbb{R}^n$ on $v^+$, the input-potential, log-marginal, raw-marginal, and edge-score adjoints of cycle $t$ are
\begin{align}
  \eta^- &= M_t^\top \eta^+, &
  \lambda_b &= \eta^+, &
  \lambda_a &= -C_t \eta^+, \label{eq:source-log-marginal-adjoints}\\
  g^a &= -(C_t \eta^+) \oslash a, &
  g^b &= \eta^+ \oslash b, &
  g^s_{ij} &= (R_t)_{ij}(C_t \eta^+)_i - (C_t)_{ij}\eta^+_j \quad ((i,j)\in\Omega). \label{eq:score-source-adjoint}
\end{align}
For logits $\theta$ at temperature $\varepsilon$, $g^\theta_{ij} = g^s_{ij}/\varepsilon$. When marginals stay compatible on each component $\Gamma = (I_\Gamma, J_\Gamma)$, $(\lambda_a, \lambda_b)$ is defined modulo $(a_{I_\Gamma}, -b_{J_\Gamma})$.
\end{proposition}

Because column potentials are defined modulo additive constants on a connected support, we work on the quotient $Q = \mathbb{R}^n / \mathrm{span}\{\one\}$ and its dual zero-sum hyperplane $Q^* = \{\eta \in \mathbb{R}^n : \one^\top \eta = 0\}$ equipped with the total-variation norm $\|\eta\|_{\TV} = \frac{1}{2}\|\eta\|_1$. Since $M_t\one = \one$, its transpose maps $Q^*$ into itself. When surrounding differentiable layers inject direct non-transport gradients at step $t$ (through scores, marginals, or residual paths), projecting those upstream terms onto $Q^*$ via $\Pi_0 x = x - (\one^\top x / n)\one$ yields an aggregate source vector $\xi_t \in Q^*$ and the inhomogeneous backward recurrence
\begin{equation}
\label{eq:tail-recurrence}
  \eta_t = M_t^\top \eta_{t+1} + \xi_t, \qquad t = 1, \dots, K,
\end{equation}
with terminal zero-sum adjoint $\eta_{K+1} \in Q^*$. Complete proofs and the formal componentwise quotient lift are given in Appendix~\ref{app:calculus-proofs}.

\subsection{Dobrushin Contraction for Unrolled and Implicit Adjoints}

For any row-stochastic matrix $A \in \mathbb{R}^{n\times n}$, the Dobrushin contraction coefficient \citep{dobrushin1956central,gaubert2013dobrushin,levin2017markov} measures the worst-case pairwise discrepancy between rows:
\[
  \tau(A) = \frac{1}{2} \max_{1 \le p,q \le n} \sum_{k=1}^n |A_{pk} - A_{qk}|
  = 1 - \min_{p,q} \sum_{k=1}^n \min\{A_{pk}, A_{qk}\} \in [0, 1].
\]
Thus $\tau(A) < 1$ exactly when every pair of rows shares positive mass.

\begin{theorem}[Dobrushin contraction for unrolled and implicit differentiation]
\label{thm:dobrushin-contraction}
Let $A, M, M_1, \dots, M_K \in \mathbb{R}^{n\times n}$ be row-stochastic matrices.
\begin{enumerate}[leftmargin=1.5em]
  \item \textbf{One-step and multi-step contraction.} For every $\eta \in Q^*$,
  \[
    \|A^\top \eta\|_{\TV} \le \tau(A)\|\eta\|_{\TV},
    \qquad
    \|(M_K \cdots M_1)^\top \eta\|_{\TV} \le \tau(M_K \cdots M_1)\|\eta\|_{\TV} \le \Bigl(\prod_{t=1}^K \tau(M_t)\Bigr)\|\eta\|_{\TV}.
  \]
  \item \textbf{Unrolled tail bound.} If $\eta_1, \dots, \eta_{K+1} \in Q^*$ satisfy \eqref{eq:tail-recurrence} with $\xi_t \in Q^*$, then
  \begin{equation}
  \label{eq:inhomogeneous-tail-bound}
    \|\eta_1\|_{\TV} \le \tau(M_K \cdots M_1)\|\eta_{K+1}\|_{\TV}
    + \sum_{t=1}^K \tau(M_{t-1}\cdots M_1)\|\xi_t\|_{\TV}.
  \end{equation}
  \item \textbf{Implicit adjoint and truncation error.} If $\tau(M) < 1$ and $\xi \in Q^*$, then $(I - M^\top)\eta = \xi$ has a unique solution $\eta = \sum_{k=0}^\infty (M^\top)^k \xi \in Q^*$ with $\|\eta\|_{\TV} \le \|\xi\|_{\TV}/(1-\tau(M))$, and the $K$-step Neumann partial sum $\eta^{(K)} = \sum_{k=0}^{K-1} (M^\top)^k \xi$ satisfies
  \[
    \|\eta - \eta^{(K)}\|_{\TV}
    \le \frac{\tau(M^K)}{1-\tau(M)}\|\xi\|_{\TV}
    \le \frac{\tau(M)^K}{1-\tau(M)}\|\xi\|_{\TV}.
  \]
\end{enumerate}
\end{theorem}

The proof decomposes any zero-sum vector $\eta = \eta^+ - \eta^-$ into normalized positive and negative probability measures and applies standard Markov TV contraction (Appendix~\ref{app:calculus-proofs}).

\section{From Graph Support to Score-Uniform Contraction}
\label{sec:baselines}
\label{sec:face-lattice}

We now determine when a sparse support mask $\Omega$ and marginals $(a,b)$ guarantee $\tau(M) < 1$ at a balanced fixed point $P \in \Pi(a,b;\Omega)$. All full proofs for this section appear in Appendix~\ref{app:face-proofs}.

\subsection{Equal-Score Baselines and Score Sensitivity}

Expanding $M = C^\top R$ shows directly how the sparsity pattern $\Omega$ shapes two-step transitions.

\begin{proposition}[Two-hop support of $M$]
\label{prop:two-hop-support}
At any balanced fixed point $P \in \Pi(a,b;\Omega)$ with positive marginals,
\begin{equation}
\label{eq:kernel-entrywise}
  M_{\alpha\beta} = \sum_{i=1}^m \frac{P_{i\alpha} P_{i\beta}}{b_\alpha a_i}
  \qquad (\alpha,\beta \in [n]).
\end{equation}
If $P_{ij} > 0$ on every edge $(i,j)\in\Omega$, then $M_{\alpha\beta} > 0$ if and only if columns $\alpha$ and $\beta$ share an active row $i$.
\end{proposition}

\begin{corollary}[Qualitative overlap criterion]
\label{cor:overlap-criterion}
If $P_{ij} > 0$ on every edge of $\Omega$, then $\tau(M) < 1$ if and only if every pair of columns $\alpha,\beta \in [n]$ has overlapping two-hop column neighborhoods in $\Omega$.
\end{corollary}

In the equal-score biregular baseline, $M$ reduces to a simple random walk determined solely by the support graph.

\begin{proposition}[Biregular equal-score baseline]
\label{prop:biregular-baseline}
Let $\Omega \subseteq [m]\times[n]$ have constant row degree $d_r \ge 1$ and constant column degree $d_c \ge 1$, with $E = |\Omega| = md_r = nd_c$. With uniform marginals $a_i = 1/m$, $b_\alpha = 1/n$, and equal active scores, the unique Sinkhorn plan is $P_{i\alpha} = E^{-1}\one\{(i,\alpha)\in\Omega\}$ and
\begin{equation}
\label{eq:biregular-kernel}
  M_{\alpha\beta} = \frac{N_{\alpha\beta}}{d_c d_r},
  \qquad
  N_{\alpha\beta} = \bigl|\{i \in [m] : (i,\alpha)\in\Omega \text{ and } (i,\beta)\in\Omega\}\bigr|.
\end{equation}
\end{proposition}

If non-equal scores keep every active edge weight bounded below relative to the uniform plan, contraction degrades smoothly rather than collapsing abruptly.

\begin{proposition}[Stability under bounded plan distortion]
\label{prop:plan-distortion-stability}
In the setting of Proposition~\ref{prop:biregular-baseline}, let $M_0$ be the equal-score kernel \eqref{eq:biregular-kernel}. For any feasible plan $P \in \Pi(a,b;\Omega)$ satisfying $E P_{i\alpha} \ge \kappa^{-1}$ on all active edges ($\kappa \ge 1$), we have $M(P) \ge \kappa^{-2} M_0$ entrywise and $\tau(M(P)) \le 1 - \kappa^{-2}(1 - \tau(M_0))$. More generally, if $M_t \ge \delta_t M_t^0$ with $\delta_t \in [0,1]$, then
\[
  \tau(M_K \cdots M_1) \le 1 - \Bigl(\prod_{t=1}^K \delta_t\Bigr)\bigl(1 - \tau(M_K^0 \cdots M_1^0)\bigr).
\]
\end{proposition}

Finally, not every change in edge scores changes the transport plan: row-plus-column additive offsets simply shift the dual potentials.

\begin{proposition}[Additive score invariance and bipartite cycle rank]
\label{prop:additive-score-invariance}
Fix $\Omega$ and positive compatible marginals $(a,b)$ satisfying strict feasibility, with $c$ connected components and active vertex sets $I, J$. Scores of the separable form $B_{ij} = \alpha_i + \beta_j + c_\Gamma$ on each component $\Gamma$ leave the Sinkhorn plan unchanged ($P_L = P_0$ for all $S = LB$). Modulo this additive subspace, the space of score patterns that alter the plan has dimension equal to the bipartite cycle rank $|\Omega| - |I| - |J| + c$. In particular, if $\Omega$ is a forest, $\Pi(a,b;\Omega)$ is a singleton and $M$ is independent of the scores.
\end{proposition}

\subsection{The Face-Lattice Dichotomy}

When scores are not additive, scaling $S = LB$ by a large factor $L \to \infty$ concentrates the plan onto the optimal linear-programming face $\mathcal{F}_B = \arg\max_{P\in\Pi(a,b;\Omega)} \langle B, P\rangle$, converging to its maximum-entropy center $P_B^{\mathrm{ent}} \in \operatorname{relint}(\mathcal{F}_B)$ (Lemma~\ref{lem:zero-temperature-limit} in Appendix~\ref{app:face-proofs}). Any edge outside $\mathcal{F}_B$ vanishes in the limit.

For any nonempty face $F$ of $\Pi(a,b;\Omega)$, let $\Omega_F = \{(i,\alpha)\in\Omega : P_{i\alpha} > 0 \text{ for some } P \in F\}$ be its active support. We say $F$ has \emph{pairwise two-hop column overlap} if for every pair of columns $\alpha,\beta \in [n]$ there exist $\gamma \in [n]$ and rows $i,k \in [m]$ with $(i,\alpha), (i,\gamma), (k,\beta), (k,\gamma) \in \Omega_F$.

\begin{theorem}[Face-lattice characterization of score-uniform mixing]
\label{thm:face-lattice-dichotomy}
Under Assumption~\ref{ass:fixed-support-feasible} with strict feasibility on $\Omega$, the following are equivalent:
\begin{enumerate}[leftmargin=1.5em]
  \item There exists $\rho < 1$, depending only on $(a,b,\Omega)$, such that $\tau(M(P_\Omega(S;a,b))) \le \rho$ for every finite score matrix $S$ on $\Omega$.
  \item Every nonempty face $F$ of $\Pi(a,b;\Omega)$ has pairwise two-hop column overlap.
\end{enumerate}
Equivalently, if (2) fails, then for every $\varepsilon > 0$ there is a finite score matrix $S$ on $\Omega$ with $\tau(M(P_\Omega(S;a,b))) > 1 - \varepsilon$.
\end{theorem}

\begin{proof}[Proof sketch]
By Proposition~\ref{prop:two-hop-support}, $1 - \tau(M(P)) > 0$ on $\Pi(a,b;\Omega)$ exactly when $\supp(P)$ has pairwise two-hop column overlap. If every nonempty face $F$ satisfies (2), then $1 - \tau(M(P))$ is strictly positive everywhere on the compact polytope $\Pi(a,b;\Omega)$, so continuity guarantees a uniform positive minimum $\delta_* = 1 - \rho > 0$. Conversely, if some face $F$ lacks pairwise overlap, choosing an exposing score direction $B$ with $\mathcal{F}_B = F$ and taking $S = LB$ with $L \to \infty$ sends $P_\Omega(LB;a,b)$ into $\operatorname{relint}(F)$, whose limiting kernel has two disjoint rows and hence $\tau \to 1$. (See Appendix~\ref{app:face-proofs}.)
\end{proof}

\begin{figure}[t]
  \centering
  \resizebox{\linewidth}{!}{%
  \begin{tikzpicture}[>=stealth, font=\small]
    \node[anchor=west, font=\bfseries] at (-0.2, 3.68) {(a) Light vs.\ heavy hub ($n=6$)};
    \draw[rounded corners, draw=black!25, fill=black!2] (-0.2,-0.32) rectangle (7.0,3.36);

    \node[anchor=west, font=\scriptsize\bfseries, text=crimson] at (0.05, 3.04) {Light hub: bad face isolates $\{1,2\}$ and $\{4,5\}$};
    \node[draw, rectangle, fill=white, font=\scriptsize] (lh) at (0.85, 2.38) {hub $0$};
    \node[draw, circle, fill=white, minimum size=0.42cm, font=\scriptsize] (lc0) at (2.25, 2.38) {$0$};
    \node[draw, rectangle, fill=crimson!10, font=\scriptsize] (lb1) at (3.95, 2.38) {$\{1,2\}$};
    \node[draw, rectangle, fill=crimson!10, font=\scriptsize] (lb2) at (5.65, 2.38) {$\{4,5\}$};
    \draw[very thick, color=crimson, ->] (lh) -- node[above, font=\scriptsize] {$1/6$} (lc0);
    \draw[dashed, color=slate, ->] (lh.south) .. controls (1.5, 1.85) and (3.2, 1.85) .. (lb1.south);
    \node[font=\scriptsize, text=slate] at (3.45, 1.66) {edges to cols $1$--$5$ vanish as $L\to\infty$};

    \draw[draw=black!15] (0.05, 1.42) -- (6.75, 1.42);

    \node[anchor=west, font=\scriptsize\bfseries, text=navy] at (0.05, 1.16) {Heavy hub: $a_0 + b_j - 1 > 0$ forces hub mass};
    \node[draw, rectangle, fill=navy!10, font=\scriptsize] (hh) at (0.8, 0.44) {hub $0$};
    \foreach \col/\xpos in {0/2.05, 1/2.85, 2/3.65, 3/4.45, 4/5.25, 5/6.05} {
      \node[draw, circle, fill=white, minimum size=0.38cm, font=\scriptsize] (hc\col) at (\xpos, 0.44) {\col};
      \draw[thick, color=navy, ->] (hh) to[bend left=22] (hc\col);
    }
    \node[font=\scriptsize, text=navy] at (3.45, -0.06) {every $P_{0j} \ge \delta_j = 0.0467$ for all $S$};

    \node[anchor=west, font=\bfseries] at (7.55, 3.68) {(b) Adversarial score sweep ($r=1$ + hub)};
    \draw[rounded corners, draw=black!25, fill=black!2] (7.55,-0.32) rectangle (15.25,3.36);

    \draw[->] (8.50, 0.35) -- (14.75, 0.35) node[right, font=\scriptsize] {$L$};
    \draw[->] (8.50, 0.35) -- (8.50, 2.95) node[above, font=\scriptsize] {$\tau(M_L)$};

    \foreach \tv/\ty in {0.0/0.35, 0.5/1.55, 1.0/2.75} {
      \draw (8.42, \ty) -- (8.50, \ty);
      \node[left, font=\scriptsize] at (8.42, \ty) {\tv};
    }
    \foreach \lv/\lx in {0/8.50, 2/9.30, 5/10.50, 10/12.50, 15/14.50} {
      \draw (\lx, 0.27) -- (\lx, 0.35);
      \node[below, font=\scriptsize] at (\lx, 0.27) {\lv};
    }

    \draw[dashed, thick, color=navy!70] (8.50, 2.536) -- (14.50, 2.536);
    \node[anchor=east, font=\scriptsize, text=navy] at (14.45, 2.26) {Bound $1 - \gamma_h = 0.911$};

    \draw[very thick, color=crimson]
      (8.50, 1.640) -- (9.30, 2.695) -- (10.50, 2.750) -- (12.50, 2.750) -- (14.50, 2.750);
    \foreach \px/\py in {8.50/1.640, 9.30/2.695, 10.50/2.750, 12.50/2.750, 14.50/2.750} {
      \fill[crimson] (\px, \py) circle (1.9pt);
    }
    \node[anchor=west, font=\scriptsize\bfseries, text=crimson] at (9.70, 1.82) {Light hub ($a_0 = 1/6$)};

    \draw[very thick, color=navy]
      (8.50, 0.523) -- (9.30, 0.695) -- (10.50, 0.696) -- (12.50, 0.696) -- (14.50, 0.696);
    \foreach \px/\py in {8.50/0.523, 9.30/0.695, 10.50/0.696, 12.50/0.696, 14.50/0.696} {
      \node[fill=navy, minimum size=3.4pt, inner sep=0pt] at (\px, \py) {};
    }
    \node[anchor=west, font=\scriptsize\bfseries, text=navy] at (9.70, 1.00) {Heavy hub ($a_0 = 0.88$)};
  \end{tikzpicture}
  }
  \caption{\textbf{Illustration of the face-lattice dichotomy (Table~\ref{tab:face-stress}).} \textbf{(a)} Even with a universal hub row $0$, uniform marginals allow large scores to isolate disconnected components. \textbf{(b)} Under the adversarial score sweep $S = LB$, the light hub approaches $\tau(M_L) = 1$, whereas raising $a_0$ to $0.88$ forces mass across all columns and keeps $\tau(M_L) \le 0.1440$, well below the Corollary~\ref{cor:forced-hub-edges}(2) bound $0.9109$.}
  \label{fig:face-dichotomy}
\end{figure}

Figure~\ref{fig:face-dichotomy} illustrates this dichotomy on a $6\times 6$ circular band augmented with one universal hub row.

\subsection{Sufficient Conditions via Minorization and Forced Hubs}

Because checking all faces of $\Pi(a,b;\Omega)$ is exponential in $|\Omega|$, we next state two tractable lower bounds on row overlap.

\begin{proposition}[Doeblin minorization and mass-floored hubs]
\label{prop:minorization-hubs}
Let $M$ be the derivative kernel of a feasible plan $P \in \Pi(a,b;\Omega)$.
\begin{enumerate}[leftmargin=1.5em]
  \item If $M_{\alpha\beta} \ge \gamma q_\beta$ for all $\alpha,\beta$ and some probability vector $q$ ($\gamma \in [0,1]$), then $\tau(M) \le 1 - \gamma$; in particular, $\tau(M) \le 1 - \gamma_{\mathrm{shared}}$ where $\gamma_{\mathrm{shared}} = \sum_\beta \min_\alpha M_{\alpha\beta}$.
  \item If $H \subseteq [n]$ is a set of hub columns such that every column $\alpha \in [n]$ and hub $h \in H$ share a witness row $i(\alpha,h)$ with $P_{i\alpha} \ge p_{\min} > 0$ and $P_{ih} \ge p_{\min}$, then $M_{\alpha h} \ge p_{\min}^2 / (a_{\max} b_{\max})$ and $\tau(M) \le 1 - \min\{1,\, |H|\, p_{\min}^2 / (a_{\max} b_{\max})\}$, where $a_{\max} = \max_i a_i$ and $b_{\max} = \max_\alpha b_\alpha$.
\end{enumerate}
\end{proposition}

When hub rows hold more marginal budget than non-hub rows can leave behind, edge masses are forced to stay strictly positive regardless of scores.

\begin{corollary}[Marginal-forced hub edges and heavy global hubs]
\label{cor:forced-hub-edges}
Under Assumption~\ref{ass:fixed-support-feasible}, set $d_{i\alpha} = \bigl[b_\alpha - \sum_{k \in \Omega^\alpha \setminus \{i\}} a_k\bigr]_+$ on $(i,\alpha)\in\Omega$ and $0$ otherwise; then every $P \in \Pi(a,b;\Omega)$ satisfies $P_{i\alpha} \ge d_{i\alpha}$.
\begin{enumerate}[leftmargin=1.5em]
  \item If a set of hub rows $B \subseteq [m]$ covers every column pair $\alpha \neq \beta$ with some $i \in B$ having $d_{i\alpha} > 0$ and $d_{i\beta} > 0$, then every nonempty face has pairwise two-hop overlap and $\tau(M(P)) \le 1 - \gamma_0$ uniformly over $\Pi(a,b;\Omega)$, where
  \[
    \gamma_0 = \min_{\alpha\neq\beta} \max_{i\in B:\, d_{i\alpha} d_{i\beta} > 0}
    \sum_{\gamma:\, d_{i\gamma} > 0} \min\!\left\{\frac{d_{i\alpha} d_{i\gamma}}{b_\alpha a_i},\, \frac{d_{i\beta} d_{i\gamma}}{b_\beta a_i}\right\} > 0.
  \]
  \item In particular, if a single global hub row $h$ connects to every column and satisfies $\delta_j := a_h + b_j - T > 0$ for all $j \in [n]$, then $P_{hj} \ge \delta_j$ and $\tau(M(P)) \le 1 - \gamma_h$ uniformly, where
  \[
    \gamma_h = \Bigl(\min_{p \in [n]} \frac{\delta_p}{b_p}\Bigr) \frac{\sum_{\beta=1}^n \delta_\beta}{a_h} > 0.
  \]
\end{enumerate}
\end{corollary}

\section{Multi-Layer Schedules and Canonical Sparse Attention Patterns}
\label{sec:schedules}

Single layers with disconnected components or permutation-like kernels always give $\tau(M) = 1$ (Proposition~\ref{prop:one-step-obstructions} in Appendix~\ref{app:schedule-proofs}). Composing multiple sparse layers across depth, however, can mix gradients completely.

\subsection{Multi-Step Schedules and Factorized Mixing}

Consider $K$ fixed-support Sinkhorn layers $\mathcal{P}_t = \Pi(a_t, b_t; \Omega_t)$ on a common column index set $[n]$, each satisfying Assumption~\ref{ass:fixed-support-feasible} with strict feasibility. For a tuple of nonempty faces $(F_1,\dots,F_K)$, let $\mathcal{R}_t(F_t)$ denote the two-hop column relation induced by $\Omega_{F_t}$. A column $\alpha$ reaches $\gamma$ through the ordered product $M_K \cdots M_1$ when there is a chain $\alpha = \alpha_K, \alpha_{K-1}, \dots, \alpha_0 = \gamma$ with $(\alpha_t, \alpha_{t-1}) \in \mathcal{R}_t(F_t)$ for $t = K, \dots, 1$. We say $(F_1,\dots,F_K)$ has \emph{windowed pairwise product overlap} if the $K$-step reachable sets of every pair $\alpha, \beta \in [n]$ intersect.

\begin{theorem}[Multi-step face-lattice characterization]
\label{thm:multistep-face-lattice}
In the setting above, there exists $\delta_K > 0$ such that $\tau(M_K(P_K)\cdots M_1(P_1)) \le 1 - \delta_K$ for all finite score matrices $(S_1,\dots,S_K)$ if and only if every tuple of nonempty faces $(F_1,\dots,F_K)$ of $\mathcal{P}_1 \times \cdots \times \mathcal{P}_K$ has windowed pairwise product overlap.
\end{theorem}

While face-tuple enumeration is intractable for arbitrary supports, block-partition masks yield exact, score-independent kernels.

\begin{theorem}[Block-partition supports induce block-Gibbs kernels]
\label{thm:block-partition}
Let $\pi : X \to Z$ partition a finite column set $X$ with positive marginal $b$, define $(z,x)\in\Omega \iff \pi(x)=z$, and set $a_z = \sum_{x:\pi(x)=z} b_x$. Then $\Pi(a,b;\Omega)$ is a singleton: for every finite score matrix, $P_{zx} = b_x\one\{\pi(x)=z\}$ and
\[
  M_{xy} = \one\{\pi(x)=\pi(y)\}\,\frac{b_y}{a_{\pi(x)}}.
\]
Hence any sequence of block-partition Sinkhorn layers has the exact same reverse-mode kernel product as the corresponding sequence of block-Gibbs (heat-bath) conditional-expectation operators.
\end{theorem}

Because each column touches only one row, the marginals fix every edge weight uniquely. When $b$ factorizes across coordinates, sweeping through all coordinates mixes the entire column space in one pass, even though each individual layer has $\tau(M_t) = 1$.

\begin{theorem}[Exact mixing of coordinate-factorized schedules]
\label{thm:coordinate-schedule}
Let $X = X_1 \times \cdots \times X_K$ have strictly positive product marginal $b(x) = \prod_{t=1}^K \mu_t(x_t)$. For layer $t$, index rows by fibers $z \in X_{-t} = \prod_{s\neq t} X_s$, connect row $z$ to columns $x$ with $x_{-t} = z$, and set $a_t(z) = \prod_{s\neq t} \mu_s(z_s)$. Then for any finite scores, $(M_t)_{xy} = \one\{x_{-t}=y_{-t}\}\mu_t(y_t)$. If $|X_{-t}| \ge 2$, then $\tau(M_t) = 1$, yet for every permutation $\sigma$ of $\{1,\dots,K\}$,
\[
  M_{\sigma(K)} \cdots M_{\sigma(1)} = \one b^\top,
  \qquad
  \tau\bigl(M_{\sigma(K)} \cdots M_{\sigma(1)}\bigr) = 0.
\]
\end{theorem}

\begin{corollary}[Hypercube butterfly schedule]
\label{cor:butterfly-schedule}
Let $n = 2^r$ ($r \ge 2$), index columns by $x \in \{0,1\}^r$, and let layer $t \in \{1,\dots,r\}$ pair each $x$ with $x \oplus e_t$ via $n/2$ rows of marginal $2/n$ and $b_x = 1/n$. Then $M_t = \frac{1}{2}(I + F_t)$ where $F_t$ flips bit $t$, so $\tau(M_t) = 1$ for every $t$ while $\tau(M_r \cdots M_1) = 0$.
\end{corollary}

Proofs of Theorems~\ref{thm:multistep-face-lattice}--\ref{thm:coordinate-schedule} and Corollary~\ref{cor:butterfly-schedule} are in Appendix~\ref{app:schedule-proofs}.

\subsection{Canonical Sparse Attention Masks}
\label{sec:canonical-topologies}

Applying our baseline, spectral, and minorization results to standard sparse attention topologies gives three concrete design rules (proofs in Appendix~\ref{app:canonical-proofs}).

\begin{proposition}[Spectral bound and expander mixing depth]
\label{prop:spectral-bound}
In the biregular equal-score setting of Proposition~\ref{prop:biregular-baseline}, $M$ is symmetric and doubly stochastic. Let $\lambda_\star = \max\{|\lambda| : \lambda \text{ eigenvalue of } M \text{ on } \one^\perp\}$. Then for every $\eta \in Q^*$ and $K \ge 1$,
\[
  \|(M^K)^\top \eta\|_2 \le \lambda_\star^K \|\eta\|_2,
  \qquad
  \tau(M^K) \le \min\!\bigl\{1,\, \sqrt{n/2}\,\lambda_\star^K\bigr\}.
\]
In particular, if $\lambda_\star \le \rho \in (0,1)$, then $\tau(M^K) \le \varepsilon$ whenever $K \ge \lceil \log(\sqrt{n/2}/\varepsilon) / (-\log\rho) \rceil$.
\end{proposition}

\begin{proposition}[Circular local-band spectrum and diffusive mixing]
\label{prop:circular-band}
Consider the equal-score biregular baseline on $n$ circular positions where row $i$ connects to columns $i-r,\dots,i+r \pmod n$ ($0 \le r < n/2$, degree $d = 2r+1$). If $n \ge 8r + 2$, then $\tau(M) = 1$. Moreover, $M$ is circulant with eigenvalues $\lambda_0 = 1$ and
\[
  \lambda_k = \left(\frac{\sin(\pi k d / n)}{d \sin(\pi k / n)}\right)^2 \quad (1 \le k \le n-1),
  \qquad
  \lambda_\star \ge \lambda_1 = 1 - \frac{\pi^2(d^2-1)}{3n^2} + O(n^{-4}).
\]
Hence for fixed $r$ and $K = o(n^2)$, the first Fourier mode satisfies $\|(M^K)^\top \eta^{(1)}\|_2 / \|\eta^{(1)}\|_2 = \lambda_1^K \to 1$: fixed-width local bands require $\Omega(n^2)$ steps to mix global zero-sum gradients.
\end{proposition}

\begin{theorem}[Closed-form equal-score plan for local-plus-hub supports]
\label{thm:register-hub-plan}
Let $n \ge 1$ token nodes have a $d$-regular token--token support ($d \ge 1$), augmented with $r \ge 1$ global hub (or register) rows and columns connected to all $N = n+r$ nodes, while token rows connect to their $d$ local tokens and all $r$ hubs. With equal scores and $a = b = N^{-1}\one$, define
\[
  t = \frac{2d}{\sqrt{(n-r)^2 + 4rd} + (n-r)} > 0,
  \qquad
  A = \frac{1}{N(d+rt)},
  \quad B = tA,
  \quad D = t^2 A.
\]
Then the Sinkhorn plan equals $A$ on token--token edges, $B$ on token--hub edges, and $D$ on hub--hub edges, giving $\tau(M) \le 1 - \gamma_{\mathrm{reg}}$ with $\gamma_{\mathrm{reg}} = rt(d + rt^2)/(d + rt)^2$.
\end{theorem}

Away from equal scores, retaining the contraction guaranteed by Theorem~\ref{thm:register-hub-plan} requires either a realized hub mass floor (Proposition~\ref{prop:minorization-hubs}) or marginal-forced hub capacity (Corollary~\ref{cor:forced-hub-edges}), consistent with Theorem~\ref{thm:face-lattice-dichotomy}.

\section{Computational Complexity and Numerical Verification}
\label{sec:verification}

\subsection{Computational Complexity}

Let $d_i = |\Omega_i|$ be the active degree of row $i$ and $E = |\Omega|$. Summing rank-one outer products $P_{i\alpha}P_{i\beta}/(b_\alpha a_i)$ over $\Omega_i \times \Omega_i$ forms $M$ in $O(\sum_i d_i^2)$ time with $\operatorname{nnz}(M) \le \sum_i d_i^2$. Exact evaluation of $\tau(M)$ compares all row pairs in $O(n\,\operatorname{nnz}(M))$ time, reaching $O(n^3)$ when $M$ densifies (as with a universal hub row). Multi-step products $M_K\cdots M_1$ also densify quickly, and checking every face in Theorems~\ref{thm:face-lattice-dichotomy} and~\ref{thm:multistep-face-lattice} is exponential in $E$ in the worst case. Consequently, exact $\tau(M_K\cdots M_1)$ and face enumeration are small-scale diagnostics, whereas Propositions~\ref{prop:minorization-hubs}--\ref{prop:circular-band}, Corollary~\ref{cor:forced-hub-edges}, and Theorems~\ref{thm:coordinate-schedule}--\ref{thm:register-hub-plan} give analytical bounds without materializing dense $n\times n$ products.

\subsection{Numerical Verification}
\label{sec:numerical}

All checks below run deterministically in \texttt{scripts/verify\_paper\_examples.py} (NumPy, fixed seed $0$) using \texttt{src/transport\_topology.py} and \texttt{src/topology\_masks.py}.

\paragraph{Finite-cycle two-plan VJP check.}
On a $6 \times 6$ circular band ($r=1$) with nonuniform compatible marginals and Gaussian scores ($\sigma=0.4$), one unbalanced row--column cycle from a nonzero potential $v_0$ gives distinct half-step plans ($\max|P_t^r - P_t^c| = 7.216 \times 10^{-2}$). Against central finite differences ($\epsilon = 10^{-6}$), the two-plan Jacobian $M_t = C_t^\top R_t$ and the adjoints $(\eta^-, \lambda_a, \lambda_b, g^s)$ from Propositions~\ref{prop:half-step-linearization}--\ref{prop:reverse-mode-cycle} agree to within $3.905 \times 10^{-10}$ max absolute error ($2.460\times 10^{-10}$ for $M_t$, $3.905\times 10^{-10}$ for $\eta^-$, $1.979\times 10^{-10}$ for $\lambda_a$, $1.719\times 10^{-10}$ for $\lambda_b$, and $2.530\times 10^{-10}$ for $g^s$). Using $P_t^c$ for both half-steps introduces a Jacobian error of $2.928 \times 10^{-1}$.

\paragraph{Canonical sparse attention supports.}
Table~\ref{tab:canonical-supports} compares seven topologies at $n=32$ with uniform marginals and Gaussian scores ($\sigma = 0.15$; $2{,}000$ Sinkhorn steps for single-layer masks, $500$ steps per singleton partition layer; max marginal residual $\le 2.78 \times 10^{-17}$). The local band ($r=2$) is connected with diameter $4$ yet keeps $\tau(M) = 1.0000$ and $\tau(M^3) = 0.9655$. Adding $2$ global hubs to a $30$-token band lowers $\tau(M)$ to $0.8887$ and $\tau(M^3)$ to $0.6383$ with empirical shared mass $\gamma_{\mathrm{shared}} = 0.1041$ ($\tau(M) \le 1 - \gamma_{\mathrm{shared}} = 0.8959$). At the same $160$-edge budget as the local band, a random $5$-regular graph reaches $\tau(M) = 0.7992$ and $\tau(M^3) = 0.2831$, while the $32$-node butterfly mask achieves $\tau(M^3) = 0.2021$ despite $\tau(M) = 1.0000$. The $5$-layer bit-partition hypercube schedule satisfies $\tau(M_t) = 1.0000$ per layer and $\tau(M_5 \cdots M_1) = 4.44 \times 10^{-16}$, matching Corollary~\ref{cor:butterfly-schedule}.

\begin{table}[t]
  \centering
  \small
  \caption{\textbf{Dobrushin contraction and structural descriptors across canonical supports ($n=32$, $\sigma=0.15$).} $E$ is edge count, $\mathrm{diam}$ is column-overlap diameter ($33$ denotes disconnected), $\gamma_{\mathrm{shared}}$ is the empirical common mass of Proposition~\ref{prop:minorization-hubs}(1), and $\mathrm{decay}_3$ is the 3-step TV contraction ratio of a random zero-sum vector. Marginal residuals are $\le 2.78 \times 10^{-17}$.}
  \label{tab:canonical-supports}
  \begin{tabular}{@{}lrrrrrrr@{}}
    \toprule
    Support Topology ($n=32$) & $E$ & Density & $\tau(M)$ & $\tau(M^3)$ & $\mathrm{decay}_3$ & $\gamma_{\mathrm{shared}}$ & $\mathrm{diam}$ \\
    \midrule
    Dense ($32\times 32$) & $1024$ & $1.000$ & $0.0036$ & $0.0000$ & $0.0000$ & $0.9927$ & $1$ \\
    Local band ($r=2$) & $160$ & $0.156$ & $1.0000$ & $0.9655$ & $0.1823$ & $0.0000$ & $4$ \\
    Disconnected ($2$ blocks) & $512$ & $0.500$ & $1.0000$ & $1.0000$ & $0.0454$ & $0.0000$ & $33$ \\
    Random $5$-regular & $160$ & $0.156$ & $0.7992$ & $0.2831$ & $0.0265$ & $0.0000$ & $2$ \\
    Local band ($r=2$) $+ 2$ hubs & $274$ & $0.268$ & $0.8887$ & $0.6383$ & $0.1161$ & $0.1041$ & $1$ \\
    Butterfly ($n=32$) & $192$ & $0.188$ & $1.0000$ & $0.2021$ & $0.0227$ & $0.0000$ & $3$ \\
    Bit-partition schedule ($5$ layers) & $5 \times 32$ & --- & $1.0000$ & \multicolumn{4}{c}{5-layer product $\tau(M_5\cdots M_1) = 4.44 \times 10^{-16}$} \\
    \bottomrule
  \end{tabular}
\end{table}

\paragraph{Face-lattice dichotomy stress test.}
On the $6\times 6$ circular band ($r=1$) with one universal hub row ($h=0$) and adversarial score direction $B \in \{\pm 1\}$ favoring the decoupled face $\{(0,0), (3,3)\} \cup (\{1,2\}\times\{1,2\}) \cup (\{4,5\}\times\{4,5\})$, we sweep $S = LB$ over $L \in \{0, 2, 5, 10, 15\}$ (Table~\ref{tab:face-stress}, Figure~\ref{fig:face-dichotomy}). With uniform marginals ($a_i = b_j = 1/6$, where $a_h + b_j - T = -2/3 \le 0$), $\tau(M_L)$ rises from $0.537535$ ($L=0$) to $1.000000$ ($L=10, 15$). Setting $a_0 = 0.88$ with $a_{1..5} = 0.024$ and $b_j = 1/6$ gives $\delta_j = 0.046667 > 0$ and the Corollary~\ref{cor:forced-hub-edges}(2) bound $\tau(M) \le 1 - \gamma_h = 0.910909$ ($\gamma_h = 0.089091$); across all $L$, $\min_j P_{0j} \ge 0.142387 \ge \delta_j$ and $\tau(M_L) \le 0.144000$.

\begin{table}[t]
  \centering
  \small
  \caption{\textbf{Adversarial score sweep on a $6\times 6$ band ($r=1$) plus one universal hub row ($10^5$ Sinkhorn steps).} Under uniform marginals (Light hub), $\tau(M_L) \to 1$. With $a_0 = 0.88$ (Heavy hub), Corollary~\ref{cor:forced-hub-edges}(2) forces $P_{0j} \ge 0.046667$ and guarantees $\tau(M_L) \le 0.910909$ uniformly. Marginal residuals are $\le 6.86 \times 10^{-10}$.}
  \label{tab:face-stress}
  \begin{tabular}{@{}cccccc@{}}
    \toprule
    Score Scale $L$ & Light Hub $\tau(M_L)$ & Light Max Res. & Heavy Hub $\tau(M_L)$ & Heavy Max Res. & $\min_j P_{0j}$ \\
    \midrule
    $0$  & $0.537535$ & $2.78 \times 10^{-17}$ & $0.072200$ & $2.78 \times 10^{-17}$ & $0.142944$ \\
    $2$  & $0.976935$ & $1.11 \times 10^{-16}$ & $0.143916$ & $6.94 \times 10^{-18}$ & $0.142387$ \\
    $5$  & $0.999941$ & $1.43 \times 10^{-13}$ & $0.144000$ & $3.33 \times 10^{-16}$ & $0.142666$ \\
    $10$ & $1.000000$ & $6.86 \times 10^{-10}$ & $0.144000$ & $1.11 \times 10^{-16}$ & $0.142667$ \\
    $15$ & $1.000000$ & $3.12 \times 10^{-14}$ & $0.144000$ & $8.33 \times 10^{-17}$ & $0.142667$ \\
    \bottomrule
  \end{tabular}
\end{table}

\paragraph{Closed-form equal-score formulas.}
On the $n=32, r=2$ equal-score circular band, numerical eigenvalues match Proposition~\ref{prop:circular-band} to $2.11 \times 10^{-15}$ ($\lambda_\star = \lambda_1 = 0.925186$). On the $(30+2)$ local-band-plus-hub support ($d=5, r=2$), the equal-score Sinkhorn plan matches Theorem~\ref{thm:register-hub-plan} to $2.60 \times 10^{-18}$, with $\gamma_{\mathrm{reg}} = 0.062316$, empirical $\gamma_{\mathrm{shared}} = 0.127442$, and $\tau(M) = 0.872558 \le 1 - \gamma_{\mathrm{reg}} = 0.937684$.

\section{Limitations and Conclusion}
\label{sec:conclusion}

\paragraph{Limitations.}
First, our calculus assumes a fixed support $\Omega$; dynamic pruning or top-$k$ rewiring introduces non-differentiable support jumps outside the present analysis. Second, Theorem~\ref{thm:dobrushin-contraction} characterizes the Sinkhorn scaling operator itself; end-to-end network gradients also depend on the norms of surrounding projections and residual paths entering through $\xi_t$. Third, checking every face in Theorem~\ref{thm:face-lattice-dichotomy} scales exponentially with $|\Omega|$, so large architectures must rely on the sufficient conditions in Propositions~\ref{prop:minorization-hubs}--\ref{prop:circular-band}, Corollary~\ref{cor:forced-hub-edges}, and Theorems~\ref{thm:coordinate-schedule}--\ref{thm:register-hub-plan}, or empirical evaluation of $\tau(M)$ at realized plans.

\paragraph{Conclusion.}
Differentiating a fixed-support Sinkhorn cycle yields a row-stochastic Markov kernel $M_t = C_t^\top R_t$ whose transpose governs reverse-mode gradients on zero-sum dual potentials. Evaluating sparse masks by $\tau(M)$ rather than graph diameter or density alone exposes both diffusive trapping in narrow bands and score-induced collapse onto disconnected faces of the transportation polytope. Together, the face-lattice dichotomy, marginal-forced hub bounds, and multi-layer coordinate schedules provide concrete criteria for designing sparse differentiable transport layers.

\appendix

\section{Proofs and Quotient Details for Section~\ref{sec:calculus}}
\label{app:calculus-proofs}

\begin{proof}[Proof of Proposition~\ref{prop:half-step-linearization}]
Differentiating \eqref{eq:row-half-step} replaces the log-sum-exp normalizer in row $i$ with the conditional expectation under $(R_t)_{ij} = (P_t^r)_{ij}/a_i$, giving $\dot u^+ = d_a - R_t \dot v - r_t(\dot s)$. Differentiating \eqref{eq:column-half-step} similarly replaces the column normalizer with the conditional expectation under $(C_t)_{ij} = (P_t^c)_{ij}/b_j$, giving $\dot v^+ = d_b - C_t^\top \dot u^+ - c_t(\dot s)$. Substituting the first differential into the second cancels the two negative signs on the homogeneous term to give \eqref{eq:full-linearization-with-sources}. At a balanced fixed point $P_t^r = P_t^c = P$, so $R_t = R$ and $C_t = C$.
\end{proof}

\begin{proof}[Proof of Proposition~\ref{prop:reverse-mode-cycle}]
Pairing $\eta^+ \in \mathbb{R}^n$ with \eqref{eq:full-linearization-with-sources} gives
\[
  \langle \eta^+, \dot v^+ \rangle
  = \langle M_t^\top \eta^+, \dot v \rangle
  + \langle \eta^+, d_b \rangle
  - \langle C_t \eta^+, d_a \rangle
  + \sum_i (C_t \eta^+)_i \sum_{j \in \Omega_i} (R_t)_{ij}\dot s_{ij}
  - \sum_j \eta^+_j \sum_{i \in \Omega^j} (C_t)_{ij}\dot s_{ij}.
\]
Reading off the coefficients of $\dot v$, $d_b$, $d_a$, and $\dot s_{ij}$ yields $\eta^-$ and \eqref{eq:source-log-marginal-adjoints}--\eqref{eq:score-source-adjoint}. Since $d_a = \dot a \oslash a$, $d_b = \dot b \oslash b$, and $s = \theta/\varepsilon$, the chain rule gives $g^a, g^b, g^\theta$. On each connected component $\Gamma = (I_\Gamma, J_\Gamma)$, admissible log-marginal perturbations satisfy $a_{I_\Gamma}^\top d_{a, I_\Gamma} - b_{J_\Gamma}^\top d_{b, J_\Gamma} = 0$, so adding any multiple of the constraint normal $(a_{I_\Gamma}, -b_{J_\Gamma})$ to $(\lambda_a, \lambda_b)$ leaves the pairing unchanged.
\end{proof}

\paragraph{Componentwise quotients and coupling with surrounding layers.}
When $\Omega$ has multiple connected components $\Gamma$ with column sets $J_\Gamma$, the potential gauge space is $G_\Omega = \mathrm{span}\{\one_{J_\Gamma} : \Gamma\}$, giving the component quotient $Q_\Omega = \mathbb{R}^n / G_\Omega$ and dual $Q_\Omega^* = \{\eta \in \mathbb{R}^n : \one_{J_\Gamma}^\top \eta = 0 \text{ for all } \Gamma\}$. Every single-support statement holds on each component with $(Q_\Omega, Q_\Omega^*)$. For schedules whose component partition changes across layers, $Q = \mathbb{R}^n/\mathrm{span}\{\one\}$ serves as the common ambient quotient (noting that $\tau(M) = 1$ on $Q$ when $\Omega$ is disconnected because it retains inter-component modes).

Formally, let $\iota : Q \to \one^\perp$ denote the centered lift $\iota([v]) = \Pi_0 v$ with $\Pi_0 = I - n^{-1}\one\one^\top$, and let $\Phi_t : \mathbb{R}^n \times \mathcal{Z}_t \to \mathcal{S}_t$ collect all differentiable surrounding maps producing source variables at step $t$ from the centered potential and local state $z_t \in \mathcal{Z}_t$. If $\nu_t \in \mathcal{S}_t^*$ is the aggregate non-transport adjoint at step $t$ (summed over direct local branches and excluding $M_t^\top \eta_{t+1}$), then the exact adjoint returned to $Q^*$ is the partial VJP $\xi_t = \iota^* D_1\Phi_t(\iota([v_t]), z_t)^* \nu_t = \Pi_0 \psi_t \in Q^*$, where $\psi_t = D_1\Phi_t(\Pi_0 v_t, z_t)^* \nu_t$. If $\Phi_t$ is already shift-invariant, $\one^\top\psi_t = 0$ and $\Pi_0\psi_t = \psi_t$.

\begin{proof}[Proof of Theorem~\ref{thm:dobrushin-contraction}]
(1) Decompose $\eta = \eta^+ - \eta^-$ with $\eta^\pm \ge 0$ and $\one^\top \eta^+ = \one^\top \eta^- = \|\eta\|_{\TV}$. If $\|\eta\|_{\TV} = 0$, the inequality is trivial; otherwise $p = \eta^+/\|\eta\|_{\TV}$ and $q = \eta^-/\|\eta\|_{\TV}$ are probability vectors, so $\|A^\top \eta\|_{\TV} = \|\eta\|_{\TV} \cdot \frac{1}{2}\|A^\top p - A^\top q\|_1 \le \tau(A)\|\eta\|_{\TV}$ because $\tau(A)$ bounds the TV distance between pushforwards of arbitrary distributions. Applying this to the row-stochastic product $M_K \cdots M_1$ and using submultiplicativity $\tau(AB) \le \tau(A)\tau(B)$ yields the multi-step bounds.

(2) Unrolling \eqref{eq:tail-recurrence} gives $\eta_1 = (M_K \cdots M_1)^\top \eta_{K+1} + \sum_{t=1}^K (M_{t-1}\cdots M_1)^\top \xi_t$. Since each $M_t^\top$ preserves $Q^*$ and $\xi_t \in Q^*$, every summand lies in $Q^*$. Applying part (1) and the triangle inequality gives \eqref{eq:inhomogeneous-tail-bound}.

(3) By part (1), $\|(M^\top)^k \xi\|_{\TV} \le \tau(M)^k \|\xi\|_{\TV}$, so the Neumann series converges absolutely in $Q^*$ to a solution of $(I - M^\top)\eta = \xi$. Any two solutions differ by $z \in Q^*$ with $z = M^\top z$, giving $\|z\|_{\TV} \le \tau(M)\|z\|_{\TV}$ and hence $z = 0$. Summing the geometric series bounds $\|\eta\|_{\TV} \le \|\xi\|_{\TV}/(1-\tau(M))$, and $\eta - \eta^{(K)} = (M^K)^\top \eta$ gives the truncation error bound.
\end{proof}

\section{Proofs for Section~\ref{sec:face-lattice}}
\label{app:face-proofs}

\begin{proof}[Proof of Proposition~\ref{prop:two-hop-support}]
Substitute $C_{i\alpha} = P_{i\alpha}/b_\alpha$ and $R_{i\beta} = P_{i\beta}/a_i$ into $M_{\alpha\beta} = \sum_i C_{i\alpha} R_{i\beta}$. Every summand is nonnegative, so under strict edgewise positivity the sum is positive precisely when columns $\alpha$ and $\beta$ share an active row $i$.
\end{proof}

\begin{proof}[Proof of Corollary~\ref{cor:overlap-criterion}]
By definition, $\tau(M) < 1$ holds iff $\sum_\gamma \min\{M_{\alpha\gamma}, M_{\beta\gamma}\} > 0$ for every pair $\alpha,\beta$, which by Proposition~\ref{prop:two-hop-support} is equivalent to pairwise two-hop column overlap.
\end{proof}

\begin{proof}[Proof of Proposition~\ref{prop:biregular-baseline}]
Since $P_{i\alpha} = E^{-1}\one\{(i,\alpha)\in\Omega\}$ satisfies $\sum_\alpha P_{i\alpha} = d_r / E = 1/m$ and $\sum_i P_{i\alpha} = d_c / E = 1/n$ with constant log-ratio on $\Omega$, it is the unique equal-score Sinkhorn plan. Here $R_{i\alpha} = P_{i\alpha}/a_i = 1/d_r$ and $C_{i\alpha} = P_{i\alpha}/b_\alpha = 1/d_c$ on active edges, so \eqref{eq:kernel-entrywise} yields $M_{\alpha\beta} = N_{\alpha\beta}/(d_c d_r)$.
\end{proof}

\begin{proof}[Proof of Proposition~\ref{prop:plan-distortion-stability}]
With uniform marginals, \eqref{eq:kernel-entrywise} gives $M_{\alpha\beta} = (d_c d_r)^{-1}\sum_{i:(i,\alpha),(i,\beta)\in\Omega} q_{i\alpha} q_{i\beta} \ge \kappa^{-2} (M_0)_{\alpha\beta}$. Whenever row-stochastic matrices satisfy $A \ge \delta B$ entrywise with $\delta \in [0,1]$, writing $A = \delta B + (1-\delta)N$ for a row-stochastic $N$ (when $\delta < 1$) gives $\tau(A) \le \delta\tau(B) + (1-\delta) = 1 - \delta(1-\tau(B))$. Iterating entrywise nonnegative matrix multiplication yields the multi-step bound.
\end{proof}

\begin{proof}[Proof of Proposition~\ref{prop:additive-score-invariance}]
For any $P \in \Pi(a,b;\Omega)$ and separable $B$, $\langle B, P \rangle = \alpha^\top a + \beta^\top b + \sum_\Gamma c_\Gamma m_\Gamma$, where $m_\Gamma$ is the fixed marginal mass of component $\Gamma$. Since $\langle B, P\rangle$ is constant on $\Pi(a,b;\Omega)$, maximizing $L\langle B, P\rangle + H(P)$ yields the same unique maximizer $P_0$ for every $L$, where $H(P) = -\sum_{(i,j)\in\Omega} P_{ij}\log P_{ij}$.

Conversely, on each component $\Gamma$, additive scores $\alpha_i + \beta_j + c_\Gamma$ form a subspace of dimension $|I_\Gamma| + |J_\Gamma| - 1$, summing to $|I| + |J| - c$ across components, so the quotient score space has dimension $|\Omega| - |I| - |J| + c$. If a score $B$ leaves the plan unchanged at $P_0 > 0$ on $\Omega$, the KKT optimality conditions for $H(P)$ and $\langle B, P\rangle + H(P)$ give dual potentials $(u^0, v^0)$ and $(u^B, v^B)$ satisfying
\[
  \log(P_0)_{ij} = u_i^0 + v_j^0 = B_{ij} + u_i^B + v_j^B \qquad ((i,j)\in\Omega),
\]
so $B_{ij} = (u_i^0 - u_i^B) + (v_j^0 - v_j^B)$ is additive on $\Omega$. When $\Omega$ is a forest, recursively stripping degree-1 leaves uniquely fixes every edge mass from the marginals alone.
\end{proof}

\begin{lemma}[Zero-temperature limit on the optimal face]
\label{lem:zero-temperature-limit}
Let $\mathcal{P} = \Pi(a,b;\Omega)$ satisfy Assumption~\ref{ass:fixed-support-feasible} with strict feasibility. For any finite score matrix $S$, the Sinkhorn plan $P_\Omega(S;a,b)$ is the unique maximizer of $\langle S, P\rangle + H(P)$ over $\mathcal{P}$. Given a score direction $B$ on $\Omega$ and $L > 0$, let $P_L = P_\Omega(LB;a,b)$, let $\mathcal{F}_B = \arg\max_{P\in\mathcal{P}} \langle B, P\rangle$, and let $P_B^{\mathrm{ent}}$ be the unique maximizer of $H(P)$ over $\mathcal{F}_B$. Then
\[
  \lim_{L\to\infty} P_L = P_B^{\mathrm{ent}} \in \operatorname{relint}(\mathcal{F}_B).
\]
\end{lemma}

\begin{proof}
Under strict feasibility, the strictly concave objective $\langle S, P\rangle + H(P)$ attains its unique maximum in $\operatorname{relint}(\mathcal{P})$, where the KKT conditions coincide with the Sinkhorn equations $P_{ij} = \exp(S_{ij}+u_i+v_j)$ with marginals $(a,b)$ \citep{peyre2019computational}. For $S = LB$, let $M_B = \max_{P\in\mathcal{P}} \langle B, P\rangle$. Since $H$ is bounded on the compact polytope $\mathcal{P}$ (say $|H(P) - H(P')| \le C$), optimality of $P_L$ relative to $P_B^{\mathrm{ent}} \in \mathcal{F}_B$ gives
\[
  L\langle B, P_L\rangle + H(P_L) \ge L M_B + H(P_B^{\mathrm{ent}}),
\]
so $\langle B, P_L\rangle \ge M_B - C/L$ and $H(P_L) \ge H(P_B^{\mathrm{ent}}) + L(M_B - \langle B, P_L\rangle) \ge H(P_B^{\mathrm{ent}})$. Every accumulation point $\bar P$ of $\{P_L\}$ therefore lies in $\mathcal{F}_B$ and satisfies $H(\bar P) \ge H(P_B^{\mathrm{ent}})$. By strict concavity of $H$ on the convex face $\mathcal{F}_B$, $\bar P = P_B^{\mathrm{ent}}$, so the whole family converges to $P_B^{\mathrm{ent}}$. Because $-\partial_x(x\log x) \to +\infty$ as $x \downarrow 0$, moving from any boundary point of $\mathcal{F}_B$ toward a point in $\operatorname{relint}(\mathcal{F}_B)$ increases entropy to first order, forcing $P_B^{\mathrm{ent}} \in \operatorname{relint}(\mathcal{F}_B)$.
\end{proof}

\begin{proof}[Proof of Theorem~\ref{thm:face-lattice-dichotomy}]
For $P \in \mathcal{P} = \Pi(a,b;\Omega)$, write
\[
  \Gamma(P) = 1 - \tau(M(P)) = \min_{\alpha,\beta} \sum_\gamma \min\{M(P)_{\alpha\gamma}, M(P)_{\beta\gamma}\}.
\]
Because $a$ and $b$ are fixed and positive, $M(P)$ and $\Gamma(P)$ are continuous functions of $P$ on the compact polytope $\mathcal{P}$. Every face of a standard-form transportation polytope is determined by zeroing a subset of edge variables \citep{deloera2013transportation}, so the minimal face containing $P$ has support $\supp(P)$, and by Proposition~\ref{prop:two-hop-support} we have $\Gamma(P) > 0$ if and only if $\supp(P)$ has pairwise two-hop column overlap.

If every nonempty face $F$ satisfies (2), then $\Gamma(P) > 0$ for all $P \in \mathcal{P}$. By compactness and continuity, $\delta_* = \min_{P \in \mathcal{P}} \Gamma(P) > 0$, so every finite-score Sinkhorn plan $P_\Omega(S;a,b) \in \mathcal{P}$ satisfies $\tau(M(P_\Omega(S;a,b))) \le 1 - \delta_* < 1$.

Conversely, suppose a nonempty face $F$ fails pairwise two-hop column overlap. Since every face of a polytope is exposed, there exists a score direction $B$ with $\mathcal{F}_B = F$. By Lemma~\ref{lem:zero-temperature-limit}, $P_\Omega(LB;a,b) \to P_B^{\mathrm{ent}} \in \operatorname{relint}(F)$, whose positive support is $\Omega_F$. Because $\Omega_F$ fails pairwise two-hop overlap, two rows of $M(P_B^{\mathrm{ent}})$ have disjoint support, giving $\tau(M(P_B^{\mathrm{ent}})) = 1$. Continuity of $\tau(M(\cdot))$ then implies $\tau(M(P_\Omega(LB;a,b))) \to 1$ as $L \to \infty$.
\end{proof}

\begin{proof}[Proof of Proposition~\ref{prop:minorization-hubs}]
For (1), writing each row $M_{\alpha\cdot} = \gamma q + (1-\gamma)\widetilde M_{\alpha\cdot}$ cancels $\gamma q$ in $M_{\alpha\cdot} - M_{\beta\cdot}$, bounding the TV distance by $1-\gamma$. Setting $q_\beta = (\min_\alpha M_{\alpha\beta})/\gamma_{\mathrm{shared}}$ when $\gamma_{\mathrm{shared}} > 0$ yields $\tau(M) \le 1 - \gamma_{\mathrm{shared}}$. For (2), dropping all terms except $i = i(\alpha, h)$ in \eqref{eq:kernel-entrywise} gives $M_{\alpha h} \ge p_{\min}^2 / (b_\alpha a_i) \ge p_{\min}^2 / (a_{\max} b_{\max})$. Summing this rowwise lower bound over $h \in H$ bounds $\gamma_{\mathrm{shared}}$ from below (capped at $1$ for numerical reporting), and part (1) gives the claim.
\end{proof}

\begin{proof}[Proof of Corollary~\ref{cor:forced-hub-edges}]
For any feasible plan $P$, $P_{i\alpha} = b_\alpha - \sum_{k \in \Omega^\alpha \setminus \{i\}} P_{k\alpha} \ge b_\alpha - \sum_{k \in \Omega^\alpha \setminus \{i\}} a_k$, proving $P_{i\alpha} \ge d_{i\alpha}$. Under condition (1), \eqref{eq:kernel-entrywise} yields $M_{\alpha\gamma} \ge d_{i\alpha} d_{i\gamma}/(b_\alpha a_i)$ and similarly for $\beta$, so summing $\min\{M_{\alpha\gamma}, M_{\beta\gamma}\}$ over $\gamma$ gives $1 - \tau(M(P)) \ge \gamma_0 > 0$. In case (2), non-hub rows carry total mass $T - a_h$, so $d_{hj} \ge b_j - (T - a_h) = \delta_j > 0$. Then $M_{\alpha\beta} \ge P_{h\alpha} P_{h\beta}/(b_\alpha a_h) \ge (\min_p \delta_p/b_p)\, \delta_\beta / a_h$, so Proposition~\ref{prop:minorization-hubs}(1) yields $\tau(M(P)) \le 1 - \gamma_h$.
\end{proof}

\section{Proofs for Section~\ref{sec:schedules} (Schedules)}
\label{app:schedule-proofs}

\begin{proposition}[One-step obstructions]
\label{prop:one-step-obstructions}
If $M$ has two closed components with disjoint transition supports, or if two rows of $M$ are point masses on distinct columns (as in any permutation matrix on $n \ge 2$ columns), then $\tau(M) = 1$.
\end{proposition}

\begin{proof}
Two rows with disjoint supports or point masses on distinct states have total-variation distance $1$.
\end{proof}

\begin{proof}[Proof of Theorem~\ref{thm:multistep-face-lattice}]
By Proposition~\ref{prop:two-hop-support}, for any $P_t \in \mathcal{P}_t$, the support of $M_t(P_t)$ is the two-hop relation of $\supp(P_t)$. Hence $(M_K(P_K)\cdots M_1(P_1))_{\alpha\gamma} > 0$ iff $\gamma$ is reachable from $\alpha$ along the two-hop relations of $(\supp(P_K), \dots, \supp(P_1))$. Writing $\Gamma_K(P_1,\dots,P_K) = 1 - \tau(M_K(P_K)\cdots M_1(P_1))$, we have $\Gamma_K > 0$ iff the minimal face tuple containing $(P_1,\dots,P_K)$ has windowed pairwise product overlap. If every nonempty face tuple satisfies this overlap condition, continuity of $\Gamma_K$ on the compact product polytope $\mathcal{P}_1 \times \cdots \times \mathcal{P}_K$ yields $\delta_K = \min \Gamma_K > 0$. Conversely, if a face tuple $(F_1, \dots, F_K)$ fails windowed overlap, choosing an exposing score direction $B_t$ for each $F_t$ and applying Lemma~\ref{lem:zero-temperature-limit} independently at scale $L \to \infty$ sends $M_K(P_K(L))\cdots M_1(P_1(L))$ to a stochastic matrix with two disjoint rows, driving $\tau$ to $1$.
\end{proof}

\begin{proof}[Proof of Theorem~\ref{thm:block-partition}]
Each column $x$ has a single active edge $(\pi(x), x)$, forcing $P_{\pi(x), x} = b_x$, which matches the row marginals $a_z$ automatically. Substituting this unique plan into \eqref{eq:kernel-entrywise} yields $M_{xy} = \one\{\pi(x)=\pi(y)\} b_y / a_{\pi(x)}$. Note that for non-product $b$, systematic block-Gibbs sweeps can still mix slowly when blocks are strongly correlated \citep{geman1984stochastic,levin2017markov}.
\end{proof}

\begin{proof}[Proof of Theorem~\ref{thm:coordinate-schedule}]
Applying Theorem~\ref{thm:block-partition} to $\pi_t(x) = x_{-t}$ shows the plan is unique and $(M_t)_{xy} = \one\{x_{-t}=y_{-t}\} b(y)/a_t(x_{-t}) = \one\{x_{-t}=y_{-t}\} \mu_t(y_t)$. Columns with $x_{-t} \neq y_{-t}$ have disjoint fiber supports in $M_t$, so $\tau(M_t) = 1$ whenever $|X_{-t}| \ge 2$. Since each $M_t$ independently resamples coordinate $t$ from $\mu_t$ while leaving the remaining coordinates unchanged, applying all $K$ coordinates once in any order produces distribution $b$ from every starting state, giving $M_{\sigma(K)}\cdots M_{\sigma(1)} = \one b^\top$ and $\tau = 0$.
\end{proof}

\begin{proof}[Proof of Corollary~\ref{cor:butterfly-schedule}]
Apply Theorem~\ref{thm:coordinate-schedule} with $X_t = \{0,1\}$ and uniform $\mu_t$.
\end{proof}

\section{Proofs for Canonical Sparse Attention Masks}
\label{app:canonical-proofs}

\begin{proof}[Proof of Proposition~\ref{prop:spectral-bound}]
By Proposition~\ref{prop:biregular-baseline}, $M = N/(d_c d_r)$ is symmetric ($N = N^\top$) and doubly stochastic. Orthogonalizing on $\one^\perp$ gives $\|(M^K)^\top \eta\|_2 \le \lambda_\star^K \|\eta\|_2$ for all $\eta \in Q^*$. Using $\|x\|_{\TV} = \frac{1}{2}\|x\|_1 \le \frac{\sqrt{n}}{2}\|x\|_2$, applying this to each zero-sum row difference $M^K_{\alpha\cdot} - M^K_{\beta\cdot} = (M^K)^\top (e_\alpha - e_\beta)$ with $\|e_\alpha - e_\beta\|_2 = \sqrt{2}$ yields
\[
  \tau(M^K) \le \min\!\Bigl\{1,\, \frac{\sqrt{n}}{2}\lambda_\star^K \sqrt{2}\Bigr\}
  = \min\!\bigl\{1,\, \sqrt{n/2}\,\lambda_\star^K\bigr\}.
\]
Setting $\sqrt{n/2}\,\rho^K \le \varepsilon$ and taking logarithms gives the sufficient depth $K$.
\end{proof}

\begin{proof}[Proof of Proposition~\ref{prop:circular-band}]
Row $\alpha$ of $M$ is supported on columns within circular distance at most $2r$ of $\alpha$. When $n \ge 8r + 2$, two columns exist at circular distance at least $4r + 1$, so their radius-$2r$ neighborhoods are disjoint and $\tau(M) = 1$.

Let $h \in \mathbb{R}^n$ be the indicator of $\{-r, \dots, r\} \pmod n$. By Proposition~\ref{prop:biregular-baseline}, $M_{\alpha\beta} = d^{-2}\sum_{s \in \mathbb{Z}/n\mathbb{Z}} h_s h_{s+\beta-\alpha}$ is circulant and diagonalized by the discrete Fourier basis, with eigenvalues $|\widehat h_k|^2 / d^2$ where $\widehat h_k = \sum_{s=-r}^r e^{2\pi i k s / n} = \frac{\sin(\pi k d / n)}{\sin(\pi k / n)}$ (up to phase) for $k \neq 0$. Expanding $\frac{\sin(dx)}{d\sin x} = 1 - \frac{(d^2-1)x^2}{6} + O(x^4)$ at $x = \pi/n$ and squaring gives $\lambda_1 = 1 - \frac{\pi^2(d^2-1)}{3n^2} + O(n^{-4})$. The first Fourier character $\eta^{(1)}$ is an eigenvector of $M^\top$ with eigenvalue $\lambda_1$, so $\|(M^K)^\top \eta^{(1)}\|_2 / \|\eta^{(1)}\|_2 = \lambda_1^K$, which tends to $1$ whenever $K = o(n^2)$.
\end{proof}

\begin{proof}[Proof of Theorem~\ref{thm:register-hub-plan}]
The candidate plan $P$ uses constant row and column scalings on the token class and on the hub class, so it has the Sinkhorn form for equal scores. Each token row or column has sum $dA + rB = A(d+rt) = 1/N$. Each hub row or column has sum $nB + rD = A(nt + rt^2)$. Since $rt^2 + (n-r)t - d = 0$ is equivalent to $nt + rt^2 = d + rt$, all row and column sums equal $1/N$; uniqueness of the equal-score Sinkhorn plan on the connected support confirms this is $P$.

With $a = b = N^{-1}\one$, \eqref{eq:kernel-entrywise} gives $M_{\alpha\beta} = N^2 \sum_i P_{i\alpha} P_{i\beta}$. For a token column $\alpha$ and hub column $h$, summing over the $d$ active token rows and $r$ hub rows yields
\[
  M_{\alpha h} = N^2(dAB + rBD) = \frac{t(d + r t^2)}{(d + rt)^2} = m_{\mathrm{reg}}.
\]
For a hub column $\alpha$ and hub column $h$, $M_{\alpha h} = N^2(nB^2 + rD^2) = \frac{t^2(n + r t^2)}{(d + rt)^2}$. Using $d = rt^2 + (n-r)t$, the difference $M_{\alpha h}^{\mathrm{hub}} - m_{\mathrm{reg}} = \frac{r t^2(t-1)^2}{(d+rt)^2} \ge 0$. Thus every row of $M$ puts at least $m_{\mathrm{reg}}$ on each of the $r$ hub columns, and Proposition~\ref{prop:minorization-hubs}(1) gives $\tau(M) \le 1 - r m_{\mathrm{reg}}$.
\end{proof}

\section{Branching and Residual Computation Graphs}
\label{app:dag-remark}

For a branching or residual directed acyclic computation graph in which every edge operator $M_e$ is row-stochastic and maps zero-sum adjoints to $Q^*$, suppose the homogeneous sink-to-source adjoint expands over paths $p$ from sink to source as
\[
  \eta_{\mathrm{source}} = \sum_p w_p \Bigl(\prod_{e \in p} M_e^\top\Bigr)\eta_{\mathrm{sink}},
\]
with scalar path weights $w_p$. Applying Theorem~\ref{thm:dobrushin-contraction}(1) and the triangle inequality yields
\[
  \|\eta_{\mathrm{source}}\|_{\TV}
  \le \Bigl(\sum_p |w_p| \prod_{e \in p} \tau(M_e)\Bigr)\|\eta_{\mathrm{sink}}\|_{\TV}.
\]
Any inhomogeneous source contributions injected at internal vertices must be projected onto $Q^*$ and summed with their corresponding upstream path weights, as in \eqref{eq:inhomogeneous-tail-bound}.

\bibliographystyle{plainnat}
\bibliography{references}

@inproceedings{forde2026blockwise,
  author = {Forde, Dylan B.},
  title = {{Block-Wise Differentiable Sinkhorn Attention}: Tail-Refinement Gradients with a {Single-Active-Dustbin Bridge}},
  booktitle = {Advances in Neural Information Processing Systems (NeurIPS)},
  year = {2026},
  note = {Preprint arXiv:2605.08123v2},
  url = {https://arxiv.org/abs/2605.08123}
}

@inproceedings{tay2020sparse,
  title={Sparse {Sinkhorn} Attention},
  author={Tay, Yi and Bahri, Dara and Yang, Liu and Metzler, Donald and Juan, Da-Cheng},
  booktitle={Proceedings of the 37th International Conference on Machine Learning},
  series={Proceedings of Machine Learning Research},
  volume={119},
  pages={9438--9447},
  year={2020}
}

@inproceedings{sander2022sinkformers,
  title={Sinkformers: Transformers with Doubly Stochastic Attention},
  author={Sander, Micha{\"e}l E. and Ablin, Pierre and Blondel, Mathieu and Peyr{\'e}, Gabriel},
  booktitle={Proceedings of the 25th International Conference on Artificial Intelligence and Statistics},
  series={Proceedings of Machine Learning Research},
  volume={151},
  pages={3515--3530},
  year={2022}
}

@article{shahbazi2025lotformer,
  title={{LOTFormer}: Doubly-Stochastic Linear Attention via Low-Rank Optimal Transport},
  author={Shahbazi, Ashkan and Thrash, Chayne and Bai, Yikun and Hamm, Keaton and NaderiAlizadeh, Navid and Kolouri, Soheil},
  journal={arXiv preprint arXiv:2509.23436},
  year={2025},
  doi={10.48550/arXiv.2509.23436}
}

@inproceedings{cuturi2013sinkhorn,
  title={Sinkhorn Distances: Lightspeed Computation of Optimal Transport},
  author={Cuturi, Marco},
  booktitle={Advances in Neural Information Processing Systems},
  volume={26},
  year={2013},
  url={https://papers.nips.cc/paper_files/paper/2013/hash/af21d0c97db2e27e13572cbf59eb343d-Abstract.html}
}

@article{peyre2019computational,
  title={Computational Optimal Transport: With Applications to Data Science},
  author={Peyr{\'e}, Gabriel and Cuturi, Marco},
  journal={Foundations and Trends in Machine Learning},
  volume={11},
  number={5--6},
  pages={355--607},
  year={2019},
  doi={10.1561/2200000073}
}

@article{franklin1989scaling,
  title={On the Scaling of Multidimensional Matrices},
  author={Franklin, Joel and Lorenz, Jens},
  journal={Linear Algebra and its Applications},
  volume={114--115},
  pages={717--735},
  year={1989},
  doi={10.1016/0024-3795(89)90490-4}
}

@article{knight2008sinkhorn,
  title={The {Sinkhorn--Knopp} Algorithm: Convergence and Applications},
  author={Knight, Philip A.},
  journal={SIAM Journal on Matrix Analysis and Applications},
  volume={30},
  number={1},
  pages={261--275},
  year={2008},
  doi={10.1137/060659624}
}

@article{idel2016review,
  title={A Review of Matrix Scaling and {Sinkhorn}'s Normal Form for Matrices and Positive Maps},
  author={Idel, Martin},
  journal={arXiv preprint arXiv:1609.06349},
  year={2016}
}

@inproceedings{luise2018differential,
  title={Differential Properties of {Sinkhorn} Approximation for Learning with {Wasserstein} Distance},
  author={Luise, Giulia and Rudi, Alessandro and Pontil, Massimiliano and Ciliberto, Carlo},
  booktitle={Advances in Neural Information Processing Systems},
  volume={31},
  pages={5864--5874},
  year={2018}
}

@inproceedings{eisenberger2022unified,
  title={A Unified Framework for Implicit {Sinkhorn} Differentiation},
  author={Eisenberger, Marvin and Toker, Aysim and Leal-Taix{\'e}, Laura and Bernard, Florian and Cremers, Daniel},
  booktitle={Proceedings of the IEEE/CVF Conference on Computer Vision and Pattern Recognition},
  pages={509--518},
  year={2022}
}

@inproceedings{greco2023sinkhorn,
  title={Non-Asymptotic Convergence Bounds for {Sinkhorn} Iterates and Their Gradients: A Coupling Approach},
  author={Greco, Giacomo and Noble, Maxence and Conforti, Giovanni and Durmus, Alain},
  booktitle={Proceedings of the 36th Conference on Learning Theory},
  series={Proceedings of Machine Learning Research},
  volume={195},
  pages={716--746},
  year={2023}
}

@inproceedings{thornton2023initialization,
  title={Rethinking Initialization of the {Sinkhorn} Algorithm},
  author={Thornton, James and Cuturi, Marco},
  booktitle={Proceedings of the 26th International Conference on Artificial Intelligence and Statistics},
  series={Proceedings of Machine Learning Research},
  volume={206},
  pages={8682--8698},
  year={2023}
}

@inproceedings{he2026phase,
  title={Phase Transition of the {Sinkhorn--Knopp} Algorithm},
  author={He, Kun},
  booktitle={Proceedings of the 2026 Annual ACM-SIAM Symposium on Discrete Algorithms},
  pages={4238--4284},
  year={2026},
  doi={10.1137/1.9781611978971.157}
}

@article{gaubert2013dobrushin,
  title={Dobrushin Ergodicity Coefficient for {Markov} Operators on Cones, and Beyond},
  author={Gaubert, St{\'e}phane and Qu, Zheng},
  journal={arXiv preprint arXiv:1302.5226},
  year={2013}
}

@article{cominetti1994asymptotic,
  title={Asymptotic Analysis of the Exponential Penalty Trajectory in Linear Programming},
  author={Cominetti, Roberto and San Mart{\'i}n, Jaime},
  journal={Mathematical Programming},
  volume={67},
  pages={169--187},
  year={1994},
  doi={10.1007/BF01582220}
}

@inproceedings{weed2018explicit,
  title={An Explicit Analysis of the Entropic Penalty in Linear Programming},
  author={Weed, Jonathan},
  booktitle={Proceedings of the 31st Conference on Learning Theory},
  series={Proceedings of Machine Learning Research},
  volume={75},
  pages={1841--1855},
  year={2018}
}

@article{sinkhornknopp1967,
  title={Concerning Nonnegative Matrices and Doubly Stochastic Matrices},
  author={Sinkhorn, Richard and Knopp, Paul},
  journal={Pacific Journal of Mathematics},
  volume={21},
  number={2},
  pages={343--348},
  year={1967}
}

@article{dobrushin1956central,
  title={Central Limit Theorem for Nonstationary Markov Chains. {I}},
  author={Dobrushin, R. L.},
  journal={Theory of Probability and its Applications},
  volume={1},
  number={1},
  pages={65--80},
  year={1956}
}

@book{seneta2006nonnegative,
  title={Non-negative Matrices and {Markov} Chains},
  author={Seneta, Eugene},
  publisher={Springer},
  edition={Revised},
  year={2006}
}

@book{levin2017markov,
  title = {Markov Chains and Mixing Times},
  author = {Levin, David A. and Peres, Yuval},
  edition = {Second},
  publisher = {American Mathematical Society},
  year = {2017},
  note = {With contributions by Elizabeth L. Wilmer},
  url = {https://pages.uoregon.edu/dlevin/MARKOV/}
}

@article{geman1984stochastic,
  title={Stochastic Relaxation, {Gibbs} Distributions, and the {Bayesian} Restoration of Images},
  author={Geman, Stuart and Geman, Donald},
  journal={IEEE Transactions on Pattern Analysis and Machine Intelligence},
  volume={6},
  number={6},
  pages={721--741},
  year={1984},
  doi={10.1109/TPAMI.1984.4767596}
}

@book{brualdi2006combinatorial,
  title={Combinatorial Matrix Classes},
  author={Brualdi, Richard A.},
  publisher={Cambridge University Press},
  year={2006}
}

@incollection{kleewitzgall1968faces,
  title={Facets and Vertices of Transportation Polytopes},
  author={Klee, Victor and Witzgall, Christoph},
  booktitle={Mathematics of the Decision Sciences},
  publisher={American Mathematical Society},
  pages={257--282},
  year={1968}
}

@inproceedings{vaswani2017attention,
  title={Attention Is All You Need},
  author={Vaswani, Ashish and Shazeer, Noam and Parmar, Niki and Uszkoreit, Jakob and Jones, Llion and Gomez, Aidan N. and Kaiser, Lukasz and Polosukhin, Illia},
  booktitle={Advances in Neural Information Processing Systems},
  volume={30},
  year={2017}
}

@article{flashsinkhorn2026,
  title={Flash{S}inkhorn: {IO}-Aware Entropic Optimal Transport},
  author={Ye, Felix X.-F. and Li, Xingjie and Yu, An and Chang, Ming-Ching and Chu, Linsong and Wertheimer, Davis},
  journal={arXiv preprint arXiv:2602.03067},
  year={2026},
  doi={10.48550/arXiv.2602.03067}
}

@article{mhc2025,
  title={{mHC}: Manifold-Constrained Hyper-Connections},
  author={Xie, Zhenda and Wei, Yixuan and Cao, Huanqi and Zhao, Chenggang and Deng, Chengqi and Li, Jiashi and Dai, Damai and Gao, Huazuo and Chang, Jiang and Zhao, Liang and Zhou, Shangyan and Xu, Zhean and Zhang, Zhengyan and Zeng, Wangding and Hu, Shengding and Wang, Yuqing and Yuan, Jingyang and Wang, Lean and Liang, Wenfeng},
  journal={arXiv preprint arXiv:2512.24880v1},
  year={2025},
  url = {https://arxiv.org/abs/2512.24880v1}
}

@article{child2019sparse,
  title={Generating Long Sequences with Sparse Transformers},
  author={Child, Rewon and Gray, Scott and Radford, Alec and Sutskever, Ilya},
  journal={arXiv preprint arXiv:1904.10509},
  year={2019}
}

@article{darcet2023registers,
  title={Vision Transformers Need Registers},
  author={Darcet, Timoth{\'e}e and Oquab, Maxime and Mairal, Julien and Bojanowski, Piotr},
  journal={arXiv preprint arXiv:2309.16588},
  year={2023}
}

@inproceedings{shazeer2017outrageously,
  title={Outrageously Large Neural Networks: The Sparsely-Gated Mixture-of-Experts Layer},
  author={Shazeer, Noam and Mirhoseini, Azalia and Maziarz, Krzysztof and Davis, Andy and Le, Quoc and Hinton, Geoffrey and Dean, Jeff},
  booktitle={International Conference on Learning Representations},
  year={2017}
}

@article{fedus2022switch,
  title={Switch Transformers: Scaling to Trillion Parameter Models with Simple and Efficient Sparsity},
  author={Fedus, William and Zoph, Barret and Shazeer, Noam},
  journal={Journal of Machine Learning Research},
  volume={23},
  number={120},
  pages={1--39},
  year={2022}
}

@article{beltagy2020longformer,
  title={Longformer: The Long-Document Transformer},
  author={Beltagy, Iz and Peters, Matthew E. and Cohan, Arman},
  journal={arXiv preprint arXiv:2004.05150},
  year={2020}
}

@inproceedings{zaheer2020bigbird,
  title={Big Bird: Transformers for Longer Sequences},
  author={Zaheer, Manzil and Guruganesh, Guru and Dubey, Kumar Avinava and Ainslie, Joshua and Alberti, Chris and Onta{\~n}{o}n, Santiago and Pham, Philip and Ravula, Anirudh and Wang, Qifan and Yang, Li and Ahmed, Amr},
  booktitle={Advances in Neural Information Processing Systems},
  volume={33},
  pages={17283--17297},
  year={2020}
}

@article{roy2020routing,
  title={Efficient Content-Based Sparse Attention with Routing Transformers},
  author={Roy, Aurko and Saffar, Mohammad and Vaswani, Ashish and Grangier, David},
  journal={Transactions of the Association for Computational Linguistics},
  volume={9},
  pages={53--68},
  year={2021}
}

@misc{hajij2022topological,
  howpublished = {arXiv preprint arXiv:2206.00606},
  title = {Topological Deep Learning: Going Beyond Graph Data},
  author = {Hajij, Mustafa and Zamzmi, Ghada and Papamarkou, Theodore and Miolane, Nina and Guzm{\'a}n-S{\'a}enz, Aldo and Ramamurthy, Karthikeyan Natesan and Birdal, Tolga and Dey, Tamal K. and Mukherjee, Soham and Samaga, Shreyas N. and Livesay, Neal and Walters, Robin and Rosen, Paul and Schaub, Michael T.},
  year = {2022},
  archivePrefix = {arXiv},
  eprint = {2206.00606},
  doi = {10.48550/arXiv.2206.00606},
  url = {https://arxiv.org/abs/2206.00606}
}

@misc{bodnar2022neural_sheaf,
  howpublished = {arXiv preprint arXiv:2202.04579},
  title = {Neural Sheaf Diffusion: A Topological Perspective on Heterophily and Oversmoothing in {GNNs}},
  author = {Bodnar, Cristian and Di Giovanni, Francesco and Chamberlain, Benjamin Paul and Li{\`o}, Pietro and Bronstein, Michael M.},
  year = {2022},
  archivePrefix = {arXiv},
  eprint = {2202.04579},
  doi = {10.48550/arXiv.2202.04579},
  note = {Accepted to NeurIPS 2022},
  url = {https://arxiv.org/abs/2202.04579}
}

@article{pauwels2023derivatives,
  title = {The Derivatives of {Sinkhorn--Knopp} Converge},
  author = {Pauwels, Edouard and Vaiter, Samuel},
  journal = {SIAM Journal on Optimization},
  volume = {33},
  number = {3},
  pages = {1494--1517},
  year = {2023},
  doi = {10.1137/22M1512703},
  url = {https://epubs.siam.org/doi/10.1137/22M1512703}
}

@article{doeblin1938expose,
  title={Expos{\'e} de la Th{\'e}orie des Cha{\^\i}nes Simples Constantes de {Markoff} {\`a} un Nombre Fini d'{\'E}tats},
  author={Doeblin, Wolfgang},
  journal={Revue Math{\'e}matique de l'Union Interbalkanique},
  volume={2},
  pages={77--105},
  year={1938}
}

@article{wolfowitz1963products,
  author = {Wolfowitz, Jacob},
  title = {Products of Indecomposable, Aperiodic, Stochastic Matrices},
  journal = {Proceedings of the American Mathematical Society},
  volume = {14},
  number = {5},
  pages = {733--737},
  year = {1963},
  doi = {10.1090/S0002-9939-1963-0154756-3}
}

@misc{deloera2013transportation,
  author = {De Loera, Jes{\'u}s A. and Kim, Edward D.},
  title = {Combinatorics and Geometry of Transportation Polytopes: An Update},
  year = {2013},
  howpublished = {arXiv preprint arXiv:1307.0124},
  url = {https://arxiv.org/abs/1307.0124}
}

\end{document}